\documentclass{article}

\usepackage[final]{neurips_2025}

\usepackage[utf8]{inputenc} 
\usepackage[T1]{fontenc}    
\usepackage{hyperref}       
\usepackage{url}            
\usepackage{booktabs}       
\usepackage{amsfonts}       
\usepackage{nicefrac}       
\usepackage{microtype}      
\usepackage{xcolor}         

\usepackage{hyperref}
\usepackage{url}
\usepackage{graphicx}
\usepackage{enumitem}
\usepackage{xspace}
\usepackage{comment}
\usepackage{multirow}
\usepackage{booktabs}
\usepackage{amsmath}
\usepackage{amssymb}
\usepackage{bm}
\usepackage{thmtools}
\usepackage{amsthm}
\usepackage{makecell}
\usepackage{wrapfig}

\newcommand{\mypara}[1]{\vspace{0.40em}\noindent{\bf #1.}}

\newcommand{\vswd}{\vspace{0.3em}}
\newcommand{\bit}{\vswd\begin{itemize*}}
\newcommand{\eit}{\end{itemize*}\vswd}
\newcommand{\ben}{\vswd\begin{enumerate*}}
\newcommand{\een}{\end{enumerate*}\vswd}
\newcommand{\bea}{\vspace{-0.0em}\begin{eqnarray}}
\newcommand{\eea}{\end{eqnarray}\vspace{-0.0em}}
\newcommand{\beq}{\vspace{-0.0em}\begin{equation}}
\newcommand{\eeq}{\end{equation}\vspace{-0.0em}}

\renewcommand{\bit}{\vswd\begin{compactitem}}
\renewcommand{\eit}{\end{compactitem}\vswd}
\renewcommand{\ben}{\vswd\begin{compactenum}}
\renewcommand{\een}{\end{compactenum}\vswd}

\newcommand{\method}{\texttt{EvoBrain}\xspace}
\newcommand{\timethengraph}{\textit {time-then-graph}\xspace}
\newcommand{\timeandgraph}{\textit {time-and-graph}\xspace}
\newcommand{\graphthentime}{\textit {graph-then-time}\xspace}

\newcommand{\exprlt}{\precneqq}

\newcommand{\DyGNN}{dynamic GNNs\xspace}
\newcommand{\TAG}{\textit{time-and-graph}\xspace}
\newcommand{\GTT}{\textit{graph-then-time}\xspace}
\newcommand{\TTG}{\textit{time-then-graph}\xspace}

\definecolor{myPurple}{RGB}{128,0,128}
\newtheorem{defn}{Definition}
\newtheorem{problem}{Problem}

\def\BibTeX{{\rm B\kern-.05em{\sc i\kern-.025em b}\kern-.08em
    T\kern-.1667em\lower.7ex\hbox{E}\kern-.125emX}}

\title{EvoBrain: Dynamic Multi-Channel EEG Graph Modeling for Time-Evolving Brain Networks}

\author{%
  Rikuto Kotoge\textsuperscript{1} \quad Zheng Chen\textsuperscript{1}
  \quad Tasuku Kimura\textsuperscript{1} \quad Yasuko Matsubara\textsuperscript{1} \\
  \textbf{Takufumi Yanagisawa\textsuperscript{2,3}\quad  Haruhiko Kishima\textsuperscript{2} \quad  Yasushi Sakurai\textsuperscript{1}}\\
  \vspace{-0.5em} \\
  \textsuperscript{1}SANKEN, The Univerity of Osaka, Japan \\
  \textsuperscript{2}Department of Neurosurgery, Graduate School of Medicine, The University of Osaka, Japan \\
  \textsuperscript{3}Institute for Advanced Co-Creation Studies, The University of Osaka, Japan \\
  \vspace{-0.5em} \\
  \textsuperscript{1}\texttt{\{rikuto88, chenz, tasuku, yasuko, yasushi\}@sanken.osaka-u.ac.jp} \\
  \textsuperscript{2}\texttt{\{tyanagisawa, hkishima\}@nsurg.med.osaka-u.ac.jp} \\
}

\begin{document}

\maketitle

\begin{abstract}

Dynamic GNNs, which integrate temporal and spatial features in Electroencephalography (EEG) data, have shown great potential in automating seizure detection.
However, fully capturing the underlying dynamics necessary to represent brain states, such as seizure and non-seizure, remains a non-trivial task and presents two fundamental challenges.
First, most existing dynamic GNN methods are built on temporally fixed static graphs, which fail to reflect the evolving nature of brain connectivity during seizure progression. 
Second, current efforts to jointly model temporal signals and graph structures and, more importantly, their interactions remain nascent, often resulting in inconsistent performance.
To address these challenges, we present the first theoretical analysis of these two problems, demonstrating the effectiveness and necessity of explicit dynamic modeling and time-then-graph dynamic GNN method.
Building on these insights, we propose \method, a novel seizure detection model that integrates a two-stream Mamba architecture with a GCN enhanced by Laplacian Positional Encoding, following neurological insights.
Moreover, \method incorporates explicitly dynamic graph structures, allowing both nodes and edges to evolve over time.
Our contributions include 
(a) a theoretical analysis proving the expressivity advantage of explicit dynamic modeling and \TTG over other approaches, 
(b) a novel and efficient model that significantly improves AUROC by 23\% and F1 score by 30\%, compared with the dynamic GNN baseline, and 
(c) broad evaluations of our method on the challenging early seizure prediction task.
\end{abstract}

\section{Introduction}
    \label{sec:intro}

\begin{figure}
  \centering
  \includegraphics[width=0.99\linewidth]{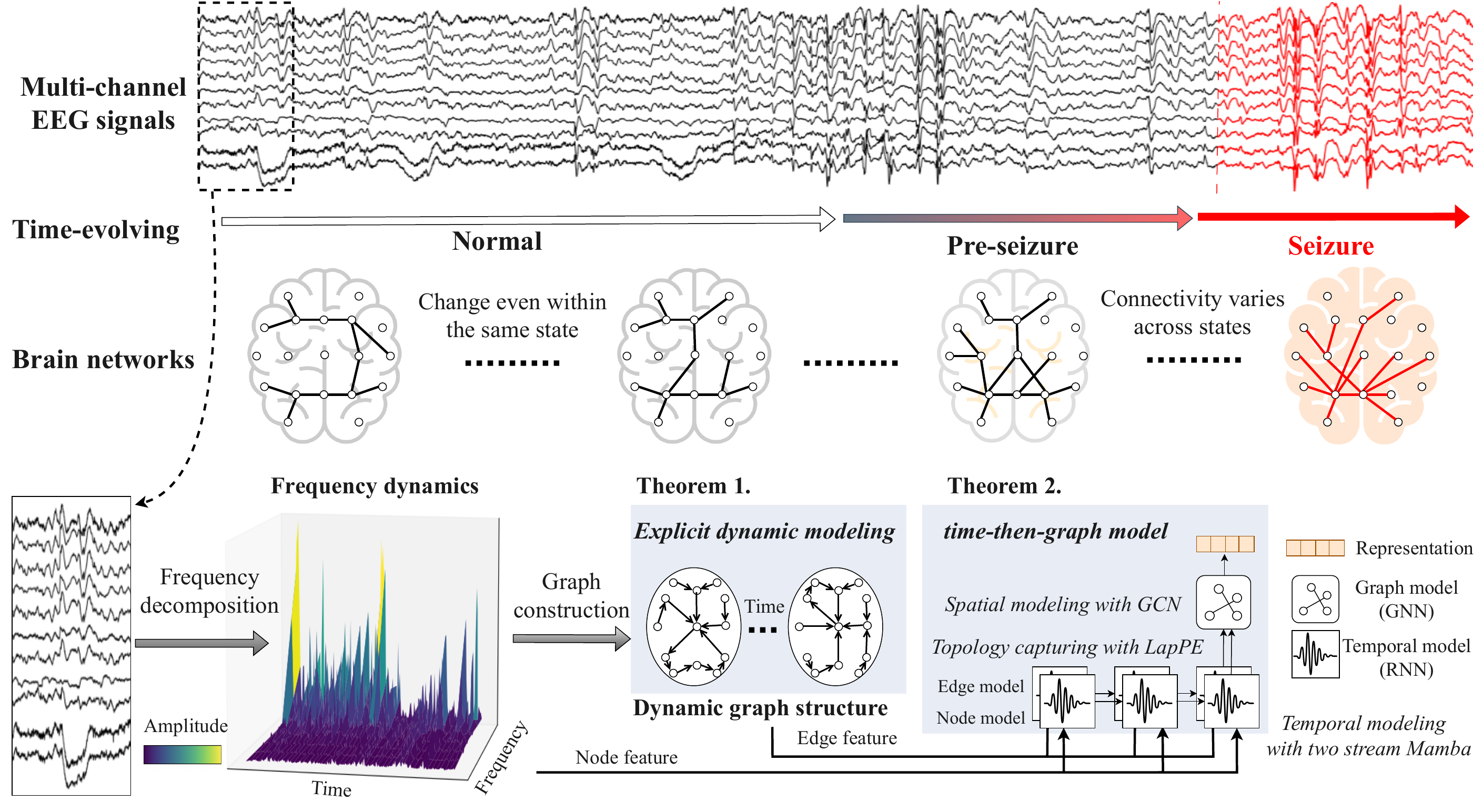}
  \caption{
  The brain network evolves over time, and changes occurring during seizures and immediately before them in the pre-seizure phase, especially within specific zones, are clinically important. These changes are captured using dynamic graphs derived from multi-channel EEG signals. 
  \method incorporates a explicit dynamic graph modeling and time-then-graph architecture.
  }
  \label{fig:story}
  \vspace{-0.1in}
\end{figure}

Epileptic seizures are typically considered a network disorder \citep{PNAS2014}.
The abnormal connections across brain regions often serve as markers of seizure events \citep{NatNeuroscience2021}. 
As such, recent studies have leveraged graph neural networks (GNNs) to model these networks for automating seizure detection \citep{MBrain_KDD23,GNN_AAAI23}.
Considering the unfolding time-course of brain dynamics, a recent trend models EEGs as a sequence of time-varying graphs.
Temporal models, such as recurrent neural networks (RNNs), integrated with GNNs, known as \DyGNN, have been proposed to learn spatiotemporal features in EEGs
\citep{graphs4mer}.
By modeling graphs at fine-grained time steps and their sequential representations, these methods can capture how graphs or brain states evolve over time, further enhancing seizure detection accuracy.

However, effectively representing brain dynamics by integrating  graph and temporal models remains non-trivial. 
In this paper, we study this critical problem in accurate seizure detection and prediction by learning dynamic representations from EEGs. 
In essence, we mainly face two challenges:\\
$\rm(\hspace{.18em}1\hspace{.18em})$ \textit{Representing EEG dynamics} (Problem \ref{prob:implicit_vs_explicit}).
We observe that although many existing methods are labeled as “dynamic,” they often employ static graph structures \citep{GNN_AAAI23, GNN_ICLR22}. 
Typically, these methods construct  a predefined graph based on EEG channel correlations from the initial snapshot, which is then fixed and shared across all subsequent EEG snapshots. 
This implicitly assumes that the spatial structure of brain activity remains constant over time.
In contrast, seizures induce evolving patterns of neuronal synchrony and desynchrony across different brain regions \citep{brain_dynamics, NatNeuroscience2021,SODorAAAI2025}, as shown in Figure \ref{fig:story}.\\
$\rm(\hspace{.18em}2\hspace{.18em})$ 
\textit{Effective Spatio-temporal Modeling (Problem \ref{prob:dynamic_eeg}).} 
Dynamic GNNs can generally be categorized into \TTG, \GTT, and \TAG architectures, following the taxonomy of \cite{time-then-graph} (A concise illustration is provided in Appendix Figure \ref{fig: illustration}).
The \TTG first model the temporal dynamics and then 
employ GNNs to learn spatial representations.
The \GTT first applies GNNs to each EEG snapshot independently, and then learns temporal dynamics from the resulting graph features. 
The \TAG is a recurrent GNN to capture temporal interactions between EEG snapshots before performing graph learning at each time step.
However, the independent GNNs in \GTT represent information at single time steps without accounting for dynamic interactions between time steps. 
While recurrent GNNs in \TAG capture graph interactions, they also rely heavily on the independent initial GNNs.
Overall, an effective \DyGNN for integrating temporal and graph-based representations to model brain dynamics remains poorly understood.\\

\noindent\textbf{Novelty and Contributions}:
We first analyze a theoretical foundation for designing dynamic graph structures and models that effectively represent brain dynamics.
Building upon this foundation, we propose \method, which effectively learns \textbf{Evo}lving, dynamic characteristics in \textbf{Brain} networks for accurate seizure detection and prediction.
Overall, we summarize our contributions as: 
\begin{itemize}[left=0pt]
    \item \textbf{Theoretical Analysis.}
    To tackle the first challenge, we propose and analyze dynamic graph structure that explicitly incorporate temporal EEG graph modeling (Theorem \ref{theorem:expressiveness_dynamic_strucutre}).
    To tackle the second challenge, we first theoretically analyze different \DyGNN approaches from EEG graph perspective (Theorem \ref{tgtotal}). 
    We discuss the necessarily of \DyGNN at node-level, since the node similarity measures are key factors in determining EEG graph construction.
    \item \textbf{Novel Model Design.}
    We hence propose a novel \TTG model, \method, which integrates a two-stream Mamba architecture \citep{Mamba} with a GCN enhanced by Laplacian Positional Encoding  \citep{lappe}, following neurological insights.
     \method achieves up to 23\% and 30\% improvements in AUROC and F1 scores, respectively, compared with the dynamic GNN baseline. 
Also, \method is 17\(\times\) faster than the SOTA \TAG dynamic GNN.
    \item \textbf{Broad Evaluation.}
    Unlike most seizure detection studies \citep{TSTCC_ijcai21,MBrain_KDD23,GNN_AAAI23}, we evaluate the more challenging task of seizure prediction, which aims to identify the preictal state before seizures. This is critical for early intervention in clinical settings, and \method maintains performance, with a 13.8\% improvement in AUROC.
\end{itemize}

\section{Problem Setting}

\subsection{Notations}
We define an EEG $\bm{x}$ with $N$ channels and $T$ snapshots as a graph $\mathcal{G}=(\mathcal{A},\mathcal{X})$, 
where \(\mathcal{A}=(\mathcal{V},\mathcal{E})\) and \(\mathcal{A}\in\mathbb{R}^{N\times N\times T}\) are the adjacency matrices.
\(\mathcal{V} \) and \(\mathcal{E}\) represent the channels (i.e., nodes) and edges.
Notably, most existing work constructs a weighted adjacency matrix \(\mathbf{A}\in\mathbb{R}^{N\times N}\) as fixed across $T$ meaning that all EEG snapshots share the same graph structure and only the node features $H$ are computed iteratively at each snapshot. 
In this paper, each edge \(e_{i,j,t} \in \mathcal{E}\) represents pairwise connectivity between channels \(v_{i}\) and \(v_{j}\), where $i,j \in N$. 
The feature vector \(x_{i,t} \in \mathbb{R}^d\) captures the electrical activity of $i$-th channel during the EEG snapshot at time step \(t\).
If $e_{i,j,t}$ exists,  \(\bm{a}_{i,j,t}\) is a non-zero value.
\(\bm{a}_{i,j,t} \in \mathbb{R}\) quantifies the connectivity strength between two channels for each snapshot.
To represent temporal EEG graphs, we define the embedding of node \(v_i\) at time step \(t\) as \(h^{node}_{i,t}\in \mathbb{R}^k\), which captures both the spatial connectivity information from the edges and the temporal dynamics from previous node embeddings.
The embedding of edge $e_{i,j,t}$, denoted as \(h^{edge}_{i,j,t}\in \mathbb{R}^l\), captures the temporal evolution of channel connectivity, reflecting changes in brain networks.

\subsection{Problem - Dynamic GNN Expressivity in EEG Modeling}

Brain networks in different states can manifest as distinct graph structures, as shown in Figure \ref{fig:story}. 
We treat \textit{expressivity analysis} as a graph isomorphism problem \citep{GIN_ICLR19}, where non-isomorphic EEG graphs represent different brain states, enabling the model to effectively distinguish between seizure and non-seizure graphs.
We formulate the challenge of representing EEG dynamics as the distinction between implicit and explicit dynamic graph modeling in Problem~\ref{prob:implicit_vs_explicit}. 
We further define the challenge of effective spatio-temporal modeling as the investigation of the expressivity of \GTT, \TAG, and \TTG \citep{time-then-graph} in dynamic EEG graph analysis in Problem~\ref{prob:dynamic_eeg}.

\begin{defn}[Implicit Dynamic Graph Modeling - Static Structure]
\label{def:explicit_implicit_graph_construction}
This approach fixes a single adjacency matrix 
\(\mathbf{A} \in \mathbb{R}^{N \times N}\) across all time steps,
although each node's feature vector \(\bm{x}_{:,t}\) evolves across \(t\). 
Formally, 
$\mathbf{A}_{:,:,t} \;=\; \widehat{\mathbf{A}},$
for each $t \in \{1,\dots,T\},$
where $\widehat{\mathbf{A}}$ is constant for all $t$.
Hence, the spatial connectivity among EEG channels remains unchanged, and only node features capture the dynamic aspects.
\end{defn}

\begin{defn}[Explicit Dynamic Graph Modeling - Dynamic Structure]
\label{def:explicit_graph_construction}
In contrast to the implicit setting, here both node features \(\bm{x}_{:,t}\) and the adjacency matrices 
\(\{\mathbf{A}_{:,:,t}\}_{t=1}^T\) can vary with time. 
Specifically, \(\mathbf{A}_{:,:,t}\in \mathbb{R}^{N\times N}\) at each snapshot \(t\) may be computed by a function \(f\):
$
    \mathbf{A}_{:,:,t} \;=\; f\bigl(\bm{x}_{:,t}\bigr),
    \quad \forall t \in \{1,\dots,T\},
$
thereby capturing the dynamic evolution of channel connectivity in addition to time-varying nodes.
\end{defn}

\begin{problem}[Implicit vs. Explicit Dynamic Graph Modeling]
\label{prob:implicit_vs_explicit}
We consider two approaches for capturing spatial relationships. In the \textit{implicit} (static) approach, a single adjacency matrix $\mathbf{A}$ remains fixed across all time steps, so only the node features evolve. In the \textit{explicit} (dynamic) approach, both node features and adjacency matrices $\{\mathbf{A}_{:,:,t}\}_{t=1}^T$ can change with $t$, allowing for time-varying connectivities derived from the EEG data. Our goal is to compare the expressiveness of these two approaches in capturing spatiotemporal dependencies for dynamic EEG graph analysis.
\end{problem}

\begin{defn}[Graph-then-time]
\label{def:graph-then-time}
This architecture first applies GNNs to learn spatial, graph information at each $t$ independently, followed by the temporal processing (e.g., by RNNs) of the resulting node embeddings.
The formal definition is given as:

\begin{equation}
\label{eqn:graphthentime}
\begin{aligned}
    \bm{h}^{node}_{i, t} &= \text{Cell}\bigg(
        \Big[
            \text{GNN}_\text{in}^{L}\big(
                \mathbf{X}_{:, t}, \mathbf{A}_{:, :, t}
            \big)
        \Big]_{i},
        \bm{h}^{node}_{i, t - 1}
    \bigg). 
\quad \mathbf{Z} &= \textbf{H}^{node}_{:, T},
\quad \forall i \in \mathcal{V},
\end{aligned}
\end{equation}
where $\bm{h}^{node}_{i, t} (1 \leq t \leq T)$ denotes the embedding of node $i \in \mathcal{V}$,
$\text{GNN}_\text{in}^{L}(\mathbf{X}_{:, t}, \bm{a}_{:, :, t})$ denotes graph learning on the current snapshot through $L$ layer GNNs. 
The learned embeddings, $\bm{h}^{node}_{i, t - 1}$, are then passed into the RNN cell to capture the temporal dependencies. 
The last step output $\bm{h}_{:, T}$ is considered the final representation $\mathbf{Z}$.
\end{defn}

\begin{defn}[Time-and-graph]
\label{def:time-and-graph}
This architecture alternately processes time and graph components, applying GNNs to each EEG snapshot, as formally defined by:
\begin{equation}
\label{eqn:timeandgraph}
\begin{aligned}
\bm{h}^{node}_{i, t} &= \text{Cell}\bigg(
     \Big[
        \text{GNN}_\text{in}^{L}\big(
            \textbf{X}_{:, t}, \mathbf{A}_{:, :, t}
        \big)
    \Big]_{i},
     \Big[
        \text{GNN}_\text{rc}^{L}\big(
            \textbf{H}_{:, t - 1}, \mathbf{A}_{:, :, t}
        \big)
    \Big]_{i}
\bigg), 
\quad \mathbf{Z} &= \textbf{H}^{node}_{:, T},
\quad \forall i \in \mathcal{V},
\end{aligned}
\end{equation}
where we initialize $H_{i, 0} = 0$.
$\text{GNN}_\text{in}^{L}$ encodes each $\bm{x}_{:, t}$ while
$\text{GNN}_\text{rc}^{L}$ encodes representations from historical snapshots $\textbf{H}_{:, t - 1}$, 
and RNN cell embeds evolution of those graph representations.
\end{defn}

\begin{defn}[Time-then-graph]
\label{def:time-then-graph}
This architecture first models the evolution of node and edge attributes over time and then applies a GNN to the resulting static graph for final representation:
\begin{equation}
\begin{aligned}
\label{eqn:timethengraph}
\bf{h}^\text{node}_{i}  &= \text{RNN}^\text{node}\big( \textbf{X}_{i, \leq T} \big), \quad\forall i \in \mathcal{V}, \quad
\bf{h}^\text{edge}_{i, j}  &= \text{RNN}^\text{edge}\big(
    \bm{a}_{i, j, \leq T}
\big), \quad \forall (i, j) \in \mathcal{E}, \\
    \mathbf{Z} &= \text{GNN}^{L}\big( \mathbf{H}^\text{node}, \mathbf{H}^\text{edge} \big).
\end{aligned}
\end{equation}
\timethengraph represents the evolution of $\bm{h}^\text{node}$ and $\bm{h}^\text{edge}$ using two sequential models $\text{RNN}^\text{node}$ and $\text{RNN}^\text{edge}$, 
resulting in a new (static) graph, which is then encoded by a $\text{GNN}^{L}$.
\end{defn}

\begin{problem}[Expressive Dynamic EEG Graph Architecture]
\label{prob:dynamic_eeg}
Determine which of these three architectures exhibits the highest expressiveness for dynamic EEG graph modeling, and characterize their relative representational power.
\end{problem}

\section{Theoretical Analysis for Dynamic EEG Graphs}
    \label{sec:preliminary}
In this section, we aim to provide a theoretical analysis of the two problems. 
To this end, we employ 1-Weisfeiler-Lehman (1-WL) GNNs, a standard tool for analyzing graph isomorphism, as detailed in Appendix~\ref{Tools}.
In the context of dynamic EEG analysis, an expressive model should be able to distinguish between different brain states by identifying non-isomorphic graphs.\\
\textbf{Remark.} 
For Theorem~\ref{tgtotal}, we follow the general GNN taxonomy proposed by \citet{time-then-graph}, but extend the analysis to EEGs, specifically focusing on node-level expressivity in dynamic EEG graphs. 
Their analysis partially targets edge- or structure-level representations using unattributed graphs, but it does not explicitly consider node features. 
In contrast, node features are essential in EEG analysis, as EEG graph construction typically relies on pairwise similarity between channels \citep{GNN_AAAI23, GNN_ICLR22}.
To this end, we provide formal theorems and proofs that consistently incorporate node features throughout the expressivity analysis. 
We also discuss the necessity of jointly modeling both node and edge representations from the perspective of EEG graphs in Appendix~\ref{proof_representaiton}.

\begin{restatable}{theorem}{implicitexplicit}
\label{theorem:expressiveness_dynamic_strucutre}
[Implicit $\precneqq$ Explicit Dynamic Graph Modeling.]
Explicit dynamic modeling (dynamic adjacency matrices) is strictly more expressive than implicit dynamic modeling (static graph structures) in capturing spatiotemporal dependencies in EEG signals.
\end{restatable}
\begin{proof}
Let $\mathcal{F}_{\mathrm{implicit}}$ and $\mathcal{F}_{\mathrm{explicit}}$ denote the function classes expressible by implicit and explicit dynamic models, respectively. Since an explicit model can replicate any implicit model by ignoring time variations in adjacency, it follows that $\mathcal{F}_{\mathrm{implicit}} \subseteq \mathcal{F}_{\mathrm{explicit}}$. 
To show strict separation, we construct two temporal EEG graphs that share identical node features but differ in adjacency at a single time step. An implicit model compresses adjacency into a static representation, potentially making these graphs indistinguishable, while an explicit model processes time-varying adjacency and can distinguish them. Thus, $\mathcal{F}_{\mathrm{implicit}} \neq \mathcal{F}_{\mathrm{explicit}}$, proving $\mathcal{F}_{\mathrm{implicit}} \subset \mathcal{F}_{\mathrm{explicit}}$. The full proof is provided in Appendix~\ref{proof_implicit_explicit}.
\end{proof}

\begin{restatable}{lemma}{gttgat}[\graphthentime $\precneqq$ \timeandgraph]
\label{gttgat}
\timeandgraph is strictly more expressive than \graphthentime representation
family on $\mathcal{T}_{n, T, \mathcal{H}eta}$ as long as we use
1-WL GNNs.
\end{restatable}

\begin{proof}
By Definition \ref{def:graph-then-time}, 
\(\bm{h}_{i, t - 1}\)  is passed without this additional GNN (i.e., $ \text{GNN}_\text{rc}^{L}(\cdot)$) to learn interactions between EEG snapshots.
This results in a simpler form of temporal representation compared to \timeandgraph:
$
    \bm{h}_{i, t - 1} \subseteq \left[\text{GNN}_\text{rc}^{L}\big( \textbf{H}_{:, t - 1}, \textbf{A}_{:, :, t} \big)\right]_{i}.
$
\graphthentime is a strict subset of \timeandgraph in terms of expressiveness.
\end{proof}

\begin{restatable}{lemma}{tagttg}[\timeandgraph $\precneqq$ \timethengraph]
\label{tagttg}
\timethengraph is strictly more expressive than \timeandgraph representation
family on $\mathbb{T}_{n, T, \theta}$, as \timethengraph outputs different representations, while \timeandgraph does not.

\end{restatable}

In Appendix \ref{proof_ttg}, we prove 
Lemma \ref{tagttg} using both node and edge representation perspectives (based on Lemma \ref{lemma:expressiveness_with_node_and_edge} in Appendix \ref{proof_representaiton}) to hold Theorem \ref{tgtotal}.
Notably, we provide a synthetic EEG task where any \timeandgraph representation fails, while a \timethengraph approach succeeds.
\timethengraph learns node and edge features across time steps to capture temporal dependencies.
This is done by encoding the temporal adjacency matrices $\textbf{A}_{:, :, \leq t}$ and node features $\textbf{X}_{:, \leq t}$ together, enabling the model to distinguish between graphs with distinct temporal structures. 
However, \timeandgraph handles each time step independently, leading to identical representations across time.

\begin{restatable}{theorem}{tgtotal}[Temporal EEG Graph Expressivity]
\label{tgtotal}
Based on Lemmas \ref{gttgat} and \ref{tagttg}, we conclude that \graphthentime is strictly less expressive than \timeandgraph, and \timeandgraph is strictly less expressive than \timethengraph on $\mathbb{T}_{n, T, \theta}$, when the graph representation is a 1-WL GNN:
\begin{equation}   
\begin{aligned}
\text{\graphthentime} 
&\precneqq_{\mathbb{T}_{n, T, \theta}} \text{\TAG} 
&\precneqq_{\mathbb{T}_{n, T, \theta}} \text{\TTG}.
\end{aligned}
\end{equation}
\end{restatable}

We confirm this conclusion in our experiments of both seizure detection and early prediction tasks.

\section{Proposed Method - EvoBrain}
    \label{sec:moethod}
    Based on the above analysis, this section presents our \method, which is built on top of explicit dynamic modeling and \TTG architecture to represent temporal EEG structures.

\subsection{Explicit Dynamic Brain Graph Structure}
\label{subsec:dynamicgraph}

Based on Theorem \ref{theorem:expressiveness_dynamic_strucutre}, we propose to construct EEG graphs for each snapshot instead of constructing a single static graph from the entire EEG recording. 
We first segment an EEG epoch into short time durations (i.e., snapshots) at regular intervals and compute channel correlations to construct a sequence of graph structures.
Specifically, for the $t$-th snapshot, we define the edge weight $a_{i,j,t}$ as the weighted adjacency matrix $\mathcal{A}$, computed as the absolute value of the normalized cross-correlation between nodes $v_i$ and $v_j$.
To prevent information redundancy and create sparse graphs, we rank the correlations among neighboring nodes and retain only the edges with the top-$\tau$ highest correlations.
\[
a_{i,j,t} = |x_{i,t} * x_{j,t}|,\quad  \ \text{if } v_j \in \mathcal{N}(v_i), \ \text{else } 0,
\]
where $x_{i,:,t}$ and $x_{j,:,t}$ represent  $v_i$ and $v_j$ channels of $t$-th EEG snapshot.
$*$ denotes the normalized cross-correlation operation.
$\mathcal{N}(v_i)$ denotes the set of top-$\tau$ neighbors of $v_i$ with higher correlations.
After computing this for $T$ snapshots, we obtain a sequence of directed, weighted EEG graphs $\mathbf{G}$ to represent brain networks at different time points. 
In other words, the dynamic nature of the EEG is captured by the evolving structure of these graphs over time.

\subsection{Dynamic GNN in Time-Then-Graph Model}
\label{subsec:modelarc}
Based on Theorem \ref{tgtotal}, we propose a novel \TTG model, where two-stream Mamba \cite{Mamba} learn the temporal evolution of node and edge attributes independently, followed by Graph Convolutional Networks (GCN) to capture spatial dependencies across electrodes in a static graph. 
This method effectively captures the temporal and spatial dynamics inherent in EEG data for seizure detection and prediction.
Notably, the model input is not the raw EEG signals but their \textbf{frequency} representation. 
Here, clinical seizure analysis aims to identify specific frequency oscillations and waveforms, such as spikes \citep{PostProcess1}.
To effectively capture such features, 
we apply a short-term Fourier transform (STFT) to the EEG signal, 
retaining the log amplitudes of the non-negative frequency components, following prior studies \citep{pmlr-seizure, seizurenet, GNN_ICLR22}.
The EEG frequency snapshots are then normalized using z-normalization across the training set. 
Consequently, an EEG frequency representation with a sequence of snapshots is formulated as $\displaystyle \mathcal{X} \in \mathbb{R}^{N\times T \times d }$, serving $N$ node initialization and dynamic graph construction.

\subsubsection{Temporal Modeling with Mamba}
We introduce a two-stream Mamba framework, with separate processes for node and edge attributes to capture their temporal dynamics.
Given a dynamic EEG graph $\mathcal{G}$, 
Mamba can be interpreted as a linear RNN with selective state updates, making it suitable for modeling brain dynamics that involve both short-term and long-term memory processes. 
In the brain, short-term memory maintains transient information over brief intervals, while long-term memory spans extended timescales, enabling the integration of past experiences into current processing.

For a traditional RNN processing sequence $\{\mathbf{x}_t\}_{t=1}^T$, the hidden state update is:
$\mathbf{h}_t = \sigma(\mathbf{W}\mathbf{h}_{t-1} + \mathbf{U}\mathbf{x}_t),$
where the weight matrices $\mathbf{W}$ and $\mathbf{U}$ representing synaptic connectivity between neurons are fixed across time $t$.
However, the synaptic weights are constantly changing, controlled in part by chemicals such as neuromodulators \citep{Synaptic_plasticity}.
This limits the model's ability to evolve information over long sequences \citep{NEURIPS2024_neuromodulation}.
In contrast, Mamba can be viewed as a linear RNN with input-dependent parameters for each element $\bf{x}^e$ (node feature $\bf{x}_{i,t}$ or weighted adjacency matrix $\bf{a}_{i,j,t}$):
\begin{equation}
\begin{aligned}
\Delta_t^e &= \tau_\Delta(f_\Delta^e(\mathbf{x}_t^e)), \quad
\mathbf{B}_t^e = f_B^e(\mathbf{x}_t^e), \quad
\mathbf{C}_t^e = f_C^e(\mathbf{x}_t^e), \\
\mathbf{h}_t^e &= \underbrace{(1 - \Delta_t^e\cdot\mathbf{D})}_{\text{selective forgetting}} \mathbf{h}_{t-1} + \underbrace{\Delta_t^e\cdot\mathbf{B}_t^e}_{\text{selective update}}\mathbf{x}_t^e,  \quad
\mathbf{y}_t^e = \mathbf{C}_t^e\mathbf{h}_t^e,
\end{aligned}
\label{eqn:mamba-updates}
\end{equation}
where $f_*$ are learnable projections for important frequency bands or edge connectivity. The hidden state $\mathbf{h}_t^e$ encodes the evolving neural activity, serving as a substrate for both transient (working) memory and more persistent (long-term) memory traces. 
The softplus activation $\tau_\Delta$ guarantees that the gating variable $\Delta_t^e$ remains positive, thereby modulating the trade-off between retaining previous neural states and incorporating new inputs. 
Specifically, the forgetting term 
$(1 - \Delta_t^e\cdot\mathbf{D})$
implements a selective mechanism analogous to \textbf{synaptic decay or inhibitory processes} that diminish outdated or irrelevant information. 
Conversely, the update term 
$\Delta_t^e\cdot\mathbf{B}_t^e$
mirrors \textbf{neuromodulatory gating} (e,g, via dopamine signaling) that selectively reinforces and integrates salient new information. 
The projection $\mathbf{C}_t^e$ translates the internal neural state into observable outputs.
The final hidden states $\mathbf{h}_i^\text{node} = \mathbf{y}_{i,T}^\text{node}$ and $\mathbf{h}_{ij}^\text{edge} = \mathbf{y}_{ij,T}^\text{edge}$ capture the temporal evolution of nodes and edges.

\subsubsection{Spatial Modeling with GCN and LapPE}
We adapt Graph Convolutional Network (GCN) and Laplacian Positional Encoding (LapPE) \citep{lappe} for efficiently modeling spatial dependencies in EEG signals. 

\mypara{Capturing Brain Network Topology}
Recent studies of neuroscience reveal that particular functions are closely associated with specific brain regions (e.g., Neocortex or Broca’s area) \citep{neocortex, connectivity}.
However, due to the fundamental computation of GNNs, nodes that are structurally equivalent under graph automorphism receive identical representations, effectively losing their individual identities \citep{lappe}.
In order to reflect the spatial specificity of the brain within spatial modeling, we introduce Laplacian Positional Encoding (LapPE).
Given edge feature $\mathbf{H}^{edge}$ the normalized Laplacian $\mathbf{L}$ is defined as:
$
    \mathbf{L} = \mathbf{I} - (\mathbf{D}^{-1/2}\mathbf{A'}\mathbf{D}^{-1/2})  = \mathbf{U}\mathbf{\Lambda}\mathbf{U}^\top, \quad
\mathbf{A'} = \tau_\text{edge}(f_{edge}(\mathbf{\mathbf{H}}^{edge})),    
$
where $\tau_\text{edge}$ is the softplus function, $f_{edge}$ is the learnable projection, 
$\mathbf{A'}$ is a weighted adjacency matrix $\mathbf{A'}$, and  $\mathbf{D}$ is a degree matrix, diagonal with $\mathbf{D} = \{\sum_j \mathbf{a'}_{i,j}\}_{i=1}^N$.
$\mathbf{I} \in \mathbb{R}^{N \times N}$ is the identity matrix, and $\mathbf{D}^{-1/2}$ is the element-wise inverse square root of the degree matrix $\mathbf{D}$. Performing an eigendecomposition on $\mathbf{L}$ yields.
The columns of $\mathbf{U} \in \mathbb{R}^{N \times N}$ are the eigenvectors of $\mathbf{L}$, and $\mathbf{\Lambda}$ is a diagonal matrix containing the corresponding eigenvalues. 

To capture the global geometry of the network, we select the first $K$ eigenvectors corresponding to the smallest eigenvalues. For node $i$, its Laplacian-based positional encoding $\mathbf{p}_i \in \mathbb{R}^K$ is given by
$
    \mathbf{p}_i = [\mathbf{u}_1[i], \mathbf{u}_2[i], \dots, \mathbf{u}_K[i]]^\top, \quad
    \mathbf{x}_i^{node} = [\mathbf{h}_i^\text{node}; \mathbf{p}_i],
$
where $\mathbf{u}_k[i]$ denotes the $i$-th component of the $k$-th eigenvector and $[;]$ indicates vector concatenation.

\mypara{Modeling EEG Spatial Dynamics}
We then adapt a GCN to learn spatial dependencies.
These graph embeddings capture the temporal evolution of EEG snapshots, with each snapshot reflecting the brain state at that particular time. 
Each layer of the GCN updates the node embeddings by aggregating information from neighboring nodes and edges. 
The node embeddings are updated as follows:

$
\label{eqn:gcn-layer}
\mathbf{h}_i^{(l+1)} = \sigma\left( (\mathbf{D}^{-1/2}\mathbf{A'}\mathbf{D}^{-1/2}) \mathbf{h}_j^{(l)} \mathbf{\Theta}^{(l)} \right),
$
where \( \mathbf{h}_i^{(l)} \) is the embedding of node \( i \) at layer \( l \) and \( \mathbf{h}_i^{(0)} = \mathbf{x}_i^{node}.\)
\( \mathbf{\Theta}^{(l)} \) is the learnable weight matrix at layer \( l \), and \( \sigma \) is an activation function (e.g., ReLU).
Afterward, we apply max pooling over the node embeddings, i.e., $\mathbf{Z} = \mathbf{h}_i^{(L)}$, followed by a fully connected layer and softmax activation for seizure detection and prediction tasks.

\section{Experiments and Results}
    \label{sec:exp}
    
\subsection{Experimental Setup}\label{subsec:data}

\noindent \textbf{Tasks.}
In this study, we focus on two tasks: seizure detection and seizure early prediction.

\begin{itemize}[left=0pt]
    \item \textbf{Seizure detection} is framed as a binary classification problem, where the goal is to distinguish between seizure and non-seizure EEG segments, termed epochs. 
    This task serves as the foundation for automated seizure monitoring systems.
    \item \textbf{Seizure early prediction} is the more challenging and clinically urgent task.    
    It aims to predict the onset of an epileptic seizure before it occurs. 
    Researchers typically frame this task as a classification problem \citep{PostProcess2,PostProcess3}, where the goal is to distinguish between pre-ictal EEG epochs and the normal state. 
    Accurate classification enables timely patient warnings or preemptive interventions, such as electrical stimulation, to prevent or mitigate seizures.
    The seizure prediction task is defined as a classification problem between inter-ictal (normal) and pre-ictal states \citep{PostProcess2}. 
However, there is no clear clinical definition regarding its onset or duration of pre-ictal state \citep{lopes2023removing}.
So it is defined as a fixed duration before the seizure occurrence \citep{PostProcess3}.
This duration is chosen to account for the time required for stimulation by implanted devices \citep{cook2013prediction} and to allow for seizure preparation. In this study, we define the pre-ictal state as one minute, providing adequate time for effective electrical stimulation to mitigate seizures or minimal preparation.
Data labeled as seizures were discarded, and a five-minute buffer zone around the boundary data was excluded from the analysis. The remaining data were used as the normal state.
\end{itemize}

\noindent \textbf{Datasets.}
We used the Temple University Hospital EEG Seizure dataset v1.5.2 (\textbf{TUSZ}) \citep{shah2018temple} to evaluate \method. 
Data description can be found in Appendix \ref{data}.
TUSZ is the largest public EEG seizure database, containing 5,612 EEG recordings with 3,050 annotated seizures. 
Each recording consists of 19 EEG channels. 
Additionally, we used the smaller CHB-MIT dataset, which consists of 844 hours of 22-channel scalp EEG data from 22 patients, including 163 recorded seizure episodes.
For the TUSZ dataset, we followed the official data split, in which a subset of patients is designated for testing. Similarly, for the CHB-MIT dataset, we used randomly selected 15\% of the patient's data for testing. Hyperparameters and augmentation strategies were set following prior studies and can be found in Appendix D.
\emph{Preprocessing:}
For the early prediction task, we defined the \emph{one-minute period} before a seizure as the preictal phase, implying the ability to predict seizures up to one minute in advance. 

\noindent \textbf{Baselines.}
We selected four \DyGNN studies as baselines: EvolveGCN-O \citep{AAAI20_EvolveGCN}, which follows the \GTT architecture, and a \TAG work, DCRNN \citep{GNN_ICLR22}.
\TTG approach, GRAPHS4MER \citep{graphs4mer} and GRU-GCN \cite{time-then-graph} are \TTG architectures. 
We included EEG foundation models, BIOT \citep{BIOT}, LaBraM \citep{LaBraM_iclr2024}, and EEGPT \citep{wang2024eegpt}. LaBraM and EEGPT were fine-tuned from the publicly available base model checkpoint.
We also evaluated LSTM \citep{LSTM} and CNN-LSTM \citep{CNN-LSTM} as referenced in \citet{GNN_ICLR22} and conventional methods, Support Vector Machine (SVM) and Random Forest.

\noindent \textbf{Metrics.}
We used the Area Under the Receiver Operating Characteristic Curve (AUROC) and F1 score as evaluation metrics. 
AUROC considers various threshold scenarios, providing an overall measure of the model's ability to distinguish between classes. 
The F1 score focuses on selecting the best threshold by balancing precision and recall.
All results are averaged over five runs with different random seeds.
Experimental details are provided in Appendix \ref{section: training}.

\subsection{Results}

\begin{table*}[t]
\centering
\caption{Performance comparison of TUSZ dataset on seizure detection and prediction for 12s and 60s.
The \textbf{best} and \underline{second best} results are highlighted.}
\label{tab:method_comparison}
\resizebox{\textwidth}{!}{\begin{tabular}{l|cccc|cccc}
\toprule
& \multicolumn{4}{c}{Detection} & \multicolumn{4}{c}{Prediction} \\
        \cmidrule(lr){2-5} \cmidrule(lr){6-9}
& \multicolumn{2}{c}{12s} & \multicolumn{2}{c}{60s} & \multicolumn{2}{c}{12s} & \multicolumn{2}{c}{60s} \\
        \cmidrule(lr){2-3} \cmidrule(lr){4-5} \cmidrule(lr){6-7} \cmidrule(lr){8-9}
Method             & AUROC & F1  & AUROC   & F1   & AUROC & F1   & AUROC   & F1   \\
\midrule
SVM & 0.765 ±0.004 & 0.369 ±0.007 & 0.720 ±0.017 & 0.390 ±0.019 & 0.566 ±0.016 & 0.320 ±0.037 & 0.561 ±0.025 & 0.312 ±0.029 \\
Random Forest & 0.778 ±0.004 & 0.354 ±0.005 & 0.739 ±0.031 & 0.386 ±0.038 & 0.566 ±0.032 & 0.344 ±0.037 & 0.550 ±0.037 & 0.330 ±0.037 \\
LSTM  & 0.794 ±0.006 & 0.381 ±0.019 & 0.720 ±0.014 & 0.392 ±0.012 & 0.572 ±0.024 & 0.353 ±0.011 & 0.559 ±0.026 & 0.393 ±0.025 \\
CNN-LSTM   & 0.754 ±0.009 & 0.356 ±0.008 & 0.680 ±0.007 & 0.331 ±0.016 & 0.621 ±0.018 & 0.389 ±0.010 & 0.528 ±0.020 & 0.316 ±0.028 \\
BIOT    & 0.726 ±0.016 & 0.320 ±0.018 & 0.651 ±0.024 & 0.280 ±0.013 & 0.576 ±0.019 & 0.425  ±0.016& 0.574 ±0.006 & 0.388 ±0.008 \\
LaBraM    & 0.825 ±0.003 & 0.472 ±0.013 & 0.793 ±0.002 & \underline{0.469 ±0.010} & 0.661 ±0.003 & \textbf{0.482  ±0.010}& \textbf{0.669 ±0.014} & \textbf{0.413 ±0.020} \\
EEGPT    & 0.803 ±0.005 & 0.415 ±0.011 & 0.743 ±0.003 & 0.406 ±0.012 & \underline{0.672 ±0.005} & 0.465 ±0.012& 0.610 ±0.018& 0.396 ±0.022\\
EvolveGCN   & 0.757 ±0.004 & 0.343 ±0.012 & 0.670 ±0.017 & 0.340 ±0.015 & 0.622 ±0.006 & 0.437 ±0.010 & 0.531 ±0.020 & 0.344 ±0.019 \\
DCRNN  & 0.817 ±0.008 & 0.415 ±0.039 & 0.808 ±0.014 & 0.435 ±0.019 & 0.634 ±0.021 & 0.401 ±0.024 & 0.601 ±0.031 & 0.397 ±0.029 \\
GRAPHS4MER  &0.833 ±0.005 & 0.413 ±0.017 & 0.778 ±0.021 & 0.439 ±0.012 & 0.632 ±0.021 & 0.438 ±0.022 & 0.638 ±0.025 & 0.355 ±0.031 \\
GRU-GCN  &  \underline{0.856 ±0.009} & \underline{0.505 ±0.009} & \underline{0.822 ±0.013} & 0.438 ±0.014 & 0.659 ±0.020 & 0.453 ±0.012 & 0.601 ±0.028 & 0.392 ±0.017 \\
\textbf{\method (Ours)} & \textbf{0.877 ±0.005} & \textbf{0.539 ±0.009} & \textbf{0.865 ±0.009 } & \textbf{0.483 ±0.006 } & \textbf{0.675 ±0.015 } & \underline{0.470 ±0.003 } & \underline{0.651 ±0.023} & \underline{0.401 ±0.023} \\

\bottomrule
\end{tabular}}
\end{table*}

\textbf{Main Results.} Table \ref{tab:method_comparison} presents a performance comparison of seizure detection and prediction for the TUSZ dataset using various models over 12-second and 60-second windows.
\method consistently outperforms dynamic GNN baselines and achieves competitive performance to the large-scale foundation model, LaBraM, while using \textbf{30 times fewer parameters}.
For the seizure detection in a 12-second window data, \method shows a 14\% increase in AUROC and 12\% increase in F1 score compared with LaBraM. 
For the 60-second window data, \method shows a 23\% increase in AUROC and 30\% ($0.670 \rightarrow 0.865$) increase in F1 score  ($0.340 \rightarrow 0.483$) compared with EvolveGCN. 
All \TTG models demonstrate relatively high performance, supporting our theoretical analysis and conclusion of Theorem \ref{tgtotal} in Section \ref{sec:preliminary}.

\begin{wrapfigure}[12]{r}{6.5cm}
\vspace{-0.15in}
  \centering
\includegraphics[width=1.0\linewidth]{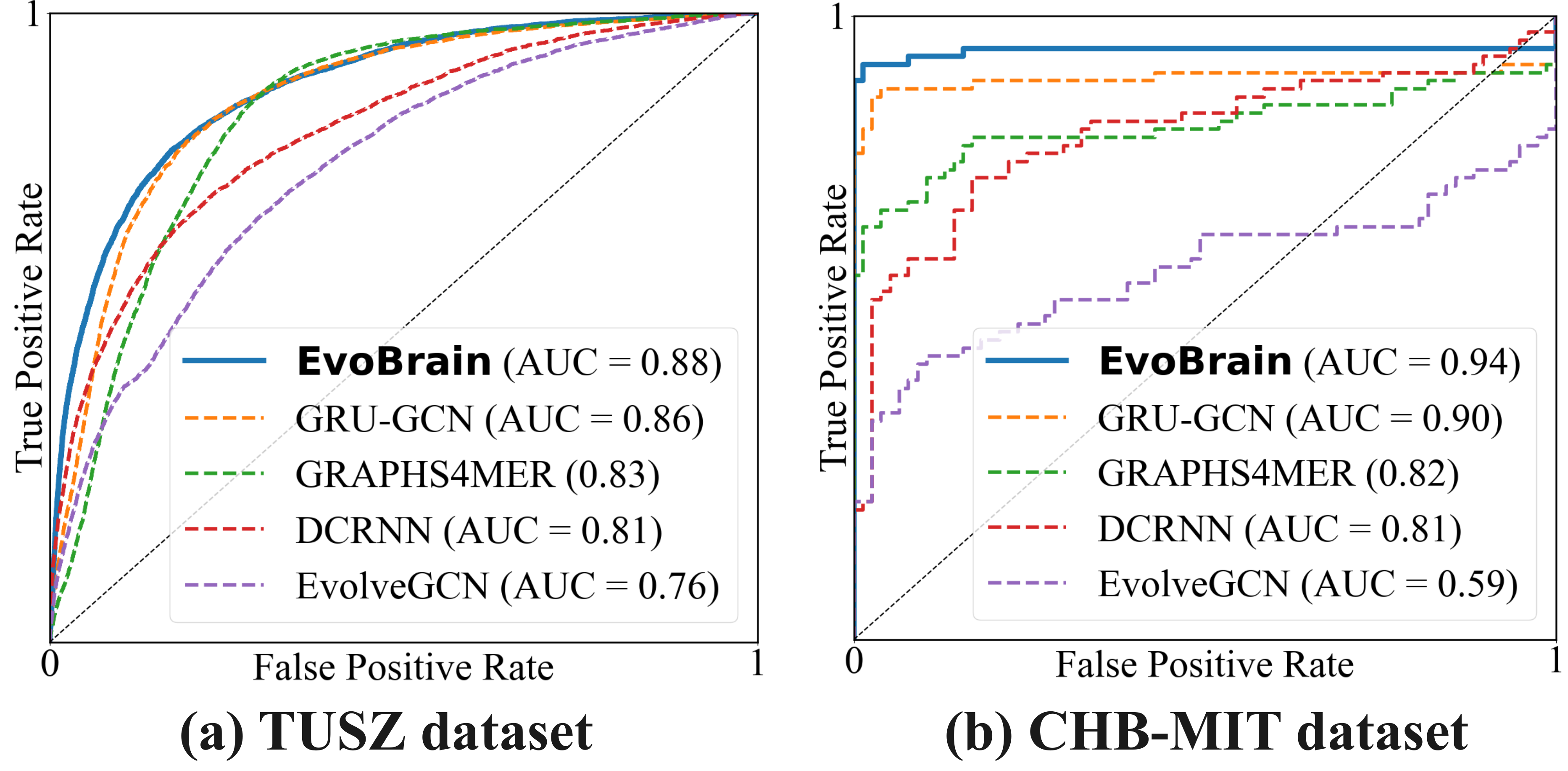}
\caption{ROC curve results for the 12-second seizure detection task on two datasets.}
  \label{fig:roc}
\end{wrapfigure}
\noindent Figure \ref{fig:roc} shows the ROC curves results comparing \method with other dynamic GNN models on seizure detection task. 
In subfigure (a), for the TUSZ dataset, \method achieves an AUC of 0.88, outperforming DCRNN (0.81) and EvolveGCN (0.76). 
Our ROC curve is positioned higher, indicating a stronger ability to differentiate between seizure and non-seizure events.
In subfigures (b) for the CHB-MIT dataset, \method achieves an AUC of 0.94, significantly higher than the 0.81 and 0.59 of \TAG and \GTT approaches, respectively.
The results show the effectiveness of \TTG for identifying seizures.

\begin{wrapfigure}[15]{r}{6.5cm}
\vspace{-0.3in}
  \centering
\includegraphics[width=1\linewidth]{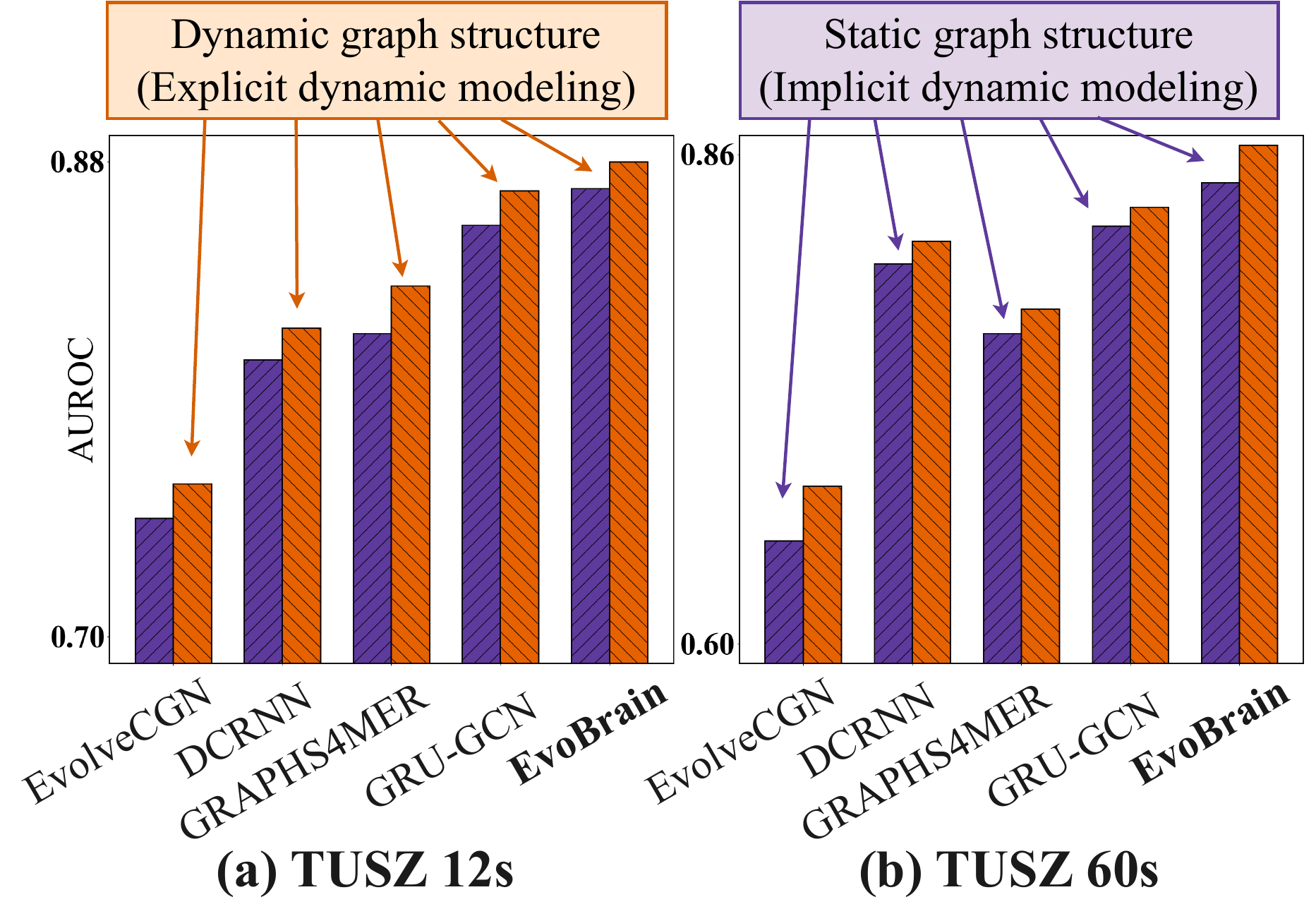}
  \caption{Comparison of the proposed dynamic graph structure and the static structure. 
  }
  \label{fig:structure}
\end{wrapfigure}

\noindent\textbf{Dynamic Graph Structure Evaluation.}
Figure \ref{fig:structure} shows the effectiveness of our proposed dynamic graph structure (i.e,. explicit dynamic modeling) compared to the static graph structure (i.e., implicit dynamic modeling) commonly used in existing work \citep{GNN_AAAI23, GNN_ICLR22}. 
The {\color{myPurple}purple} bar shows the performance of the static graph structure, while the {\color{orange}orange} bar represents the results when the static graph is replaced with our dynamic graph structures. 
As seen, the improvements are \textbf{not limited} to our \TTG method but also enhance the performance of all \DyGNN approaches. The figure highlights the effectiveness and necessity of dynamic graphs in capturing brain dynamics.
The results imply that modeling temporal dynamics in EEGs should incorporate various channel connectivity or structural information, supporting our theoretical analysis and conclusion of Theorem \ref{theorem:expressiveness_dynamic_strucutre} in Section \ref{sec:preliminary}.

\noindent\textbf{Computational Efficiency.}
To assess the computational efficiency of our method, we measured the training time and inference time.
Dynamic GNNs require computation time that scales with the length of the input data (details are provided in Appendix \ref{sec:complexity}). 
Figure \ref{fig: efficiency} (a) illustrates the average training time per step for dynamic GNNs with a batch size of 1 across various input lengths. In practice, while the RNN component operates very quickly, the GNN processing accounts for most of the computation time. 
Since our method performs GNN processing only once for each data sample, it is up to 17\(\times\) faster training time and more than 14\(\times\) faster inference time than DCRNN. 
Thus, our approach is not only superior in performance but also fast in computational efficiency.

\begin{figure}
  \centering
  \includegraphics[width=0.99\linewidth]{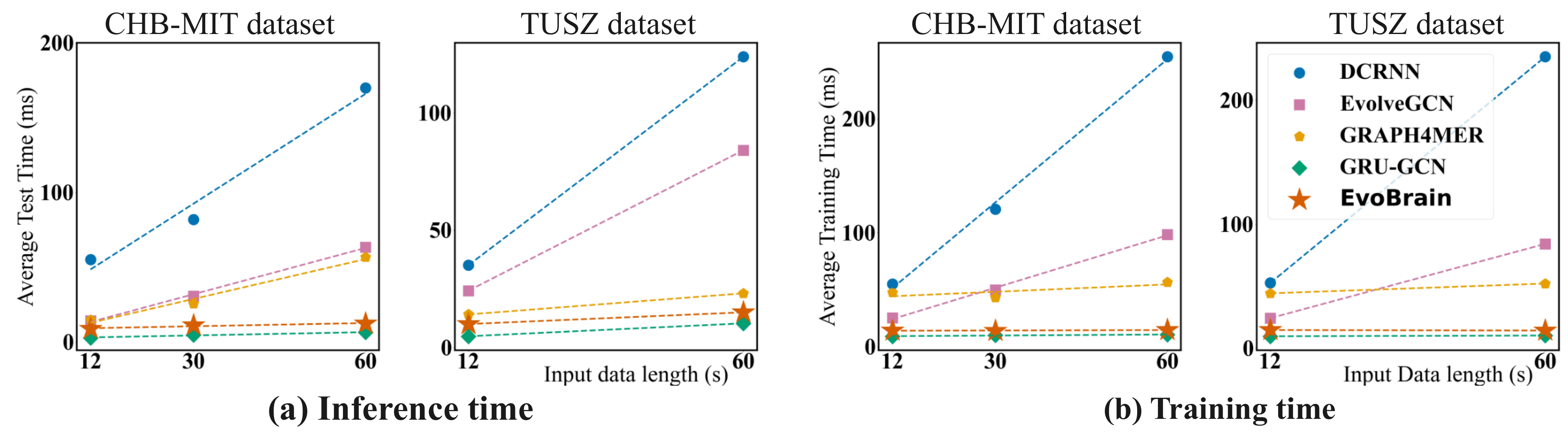}
  \caption{
  (a) Inference and (b) training time vs. input data length on CHB-MIT and TUSZ datasets. Our model achieves up to \textbf{17x faster} training times and \textbf{14x faster} inference times than its competitors, demonstrating scalability. 
  }
  \label{fig: efficiency}
    \vspace{-0.4em}
\end{figure}

\begin{figure}
  \centering
  \includegraphics[width=0.95\linewidth]{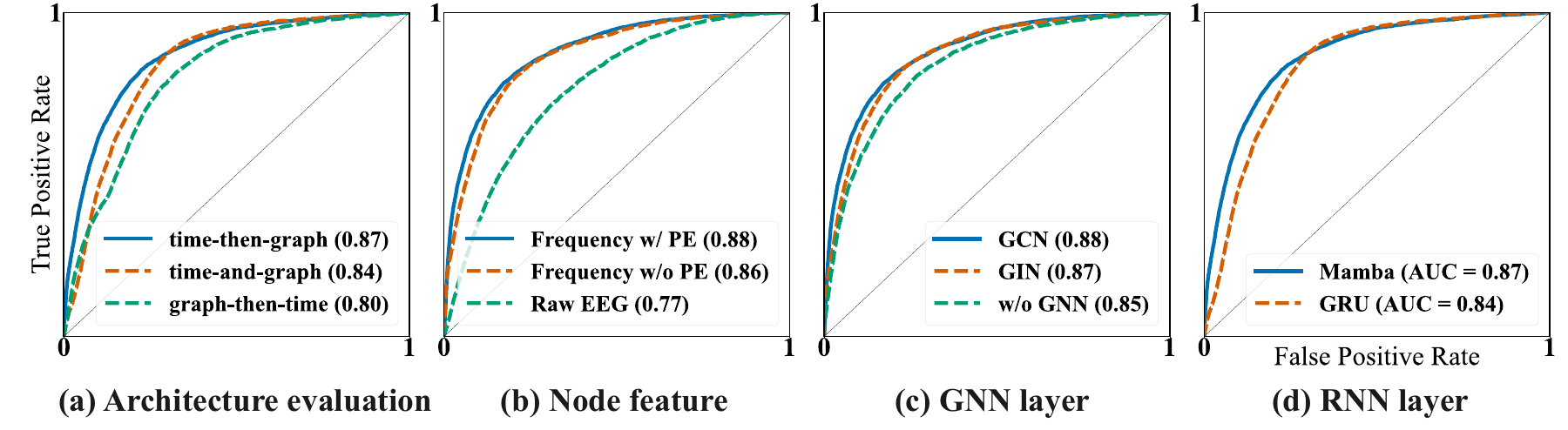}
  \caption{
   (a) Architecture evaluation using same GNN and RNN layers. (b) Results using raw EEG instead of frequency-domain features. (c) GNN and (d) RNN layer evaluation.
  }
  \label{fig: ablations}
\end{figure}

\begin{figure*}[t]
  \centering
  \includegraphics[width=1.0\linewidth]{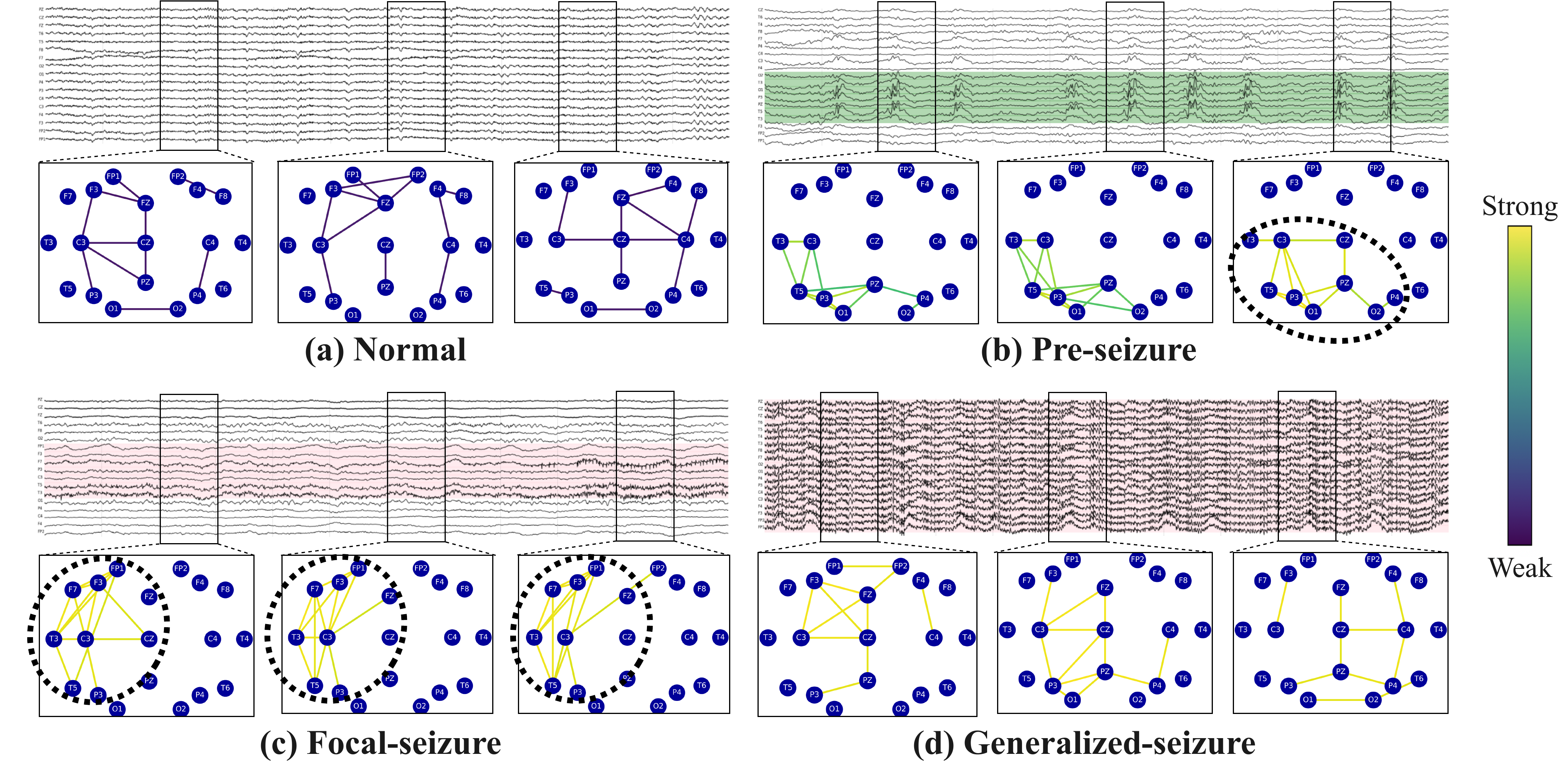}
  \caption{
  Learned graph structure visualizations. The yellow color of the edges indicates the strength of the connections. In (a) Normal state, the dark shows weak connections. In (b) Pre-seizure state, the connections in specific regions strengthen over time. In (c) Focal seizure, which occurs only in a specific area of the brain, strong connections are consistently present in a particular region. In (d) Generalized seizure, strong connections are observed across the entire brain.
  }
  \label{fig: clinical_analysis}
    \vspace{-10pt}
\end{figure*}

\paragraph{Ablation Study.}
\label{sec:ablation}
Figure \ref{fig: ablations} (a) shows the results of different architectures applied to the same GRU and GCN layers using the 12-second TUSZ dataset on the seizure detection task. Consistent with Theorem 1, the time-then-graph architecture achieved the best performance.
Figure \ref{fig: ablations} (b) shows the results when the FFT processing or LapPE was removed. 
Figure \ref{fig: ablations} (c) illustrates that replacing the GCN with GIN yields nearly identical performance, whereas removing the GNN and using only the RNN leads to a performance drop.
Figure \ref{fig: ablations} (d) shows results of replacing Mamba with GRU layers on the seizure detection task using the 60-second TUSZ dataset. 
The results demonstrate that Mamba provides performance benefits, especially for longer input sequences.

\subsection{Clinical Analysis.}
We show a case analysis of our constructed dynamic graphs from a neuroscience perspective. 
Figure \ref{fig: clinical_analysis} shows the top-10 edges with the strongest connections in the learned dynamic graphs , where the yellow color represents the strength of the connections. 
These edges are selected based on the highest values, indicating the most significant relationships captured by the model.
In Figure \ref{fig: clinical_analysis} (a), a sample unrelated to a seizure shows weak, sparse connections spread across various regions over an extended period. 
Figure \ref{fig: clinical_analysis} (b) shows a pre-seizure sample, 
where some connections gradually strengthen. 
Figure \ref{fig: clinical_analysis} (c) shows the result of a focal seizure, a type of seizure that originates in a specific area, with sustained strong connections only in a specific region.
In Figure \ref{fig: clinical_analysis} (d), a generalized seizure is illustrated, characterized by strong connections across the entire brain.\\
Successful surgical and neuromodulatory treatments critically depend on accurate localization of the seizure onset zone (SOZ) \citep{li2021neural}.
Even the most experienced clinicians are challenged because there is no clinically validated biomarker of SOZ.
Prior studies have shown that abnormal connections across several channels may constitute a more effective marker of the SOZ \citep{scharfman2007neurobiology, SOZ1, SOZ2}.
Our dynamic graph structures align with neuroscientific observations, successfully visualizing these abnormal connections and their changes. This offers promising potential for application in surgical planning and treatment strategies.

\section{Conclusion}
    \label{sec:conclusion}
    
 In this work, we propose a dynamic multi-channel EEG graph modeling approach, \method. 
 By adopting a \TTG strategy together with explicit dynamic modeling, \method captures the evolving nature of brain networks.
 Our theoretical analysis establishes the expressivity advantages of both explicit dynamic graphs and the \TTG paradigm over alternatives, providing a principled foundation for the design choices.
\method yields substantial empirical gains, including improvements of 23\% in AUROC and 30\% in F1 over a strong dynamic GNN baseline, and strong performance on early seizure prediction. Case clinical analyses further highlight potential clinical utility by revealing network changes consistent with seizure progression. Taken together, these results suggest \method is a practical and theoretically grounded step toward reliable, clinically meaningful EEG-based seizure analysis.
\paragraph{Limitations and Social Impacts.}
In seizure prediction, pre-ictal patterns are typically weaker and more spatially diffuse. While EvoBrain achieves the best performance among lightweight GNN-based models, LaBraM benefits from more parameters and pretraining on large-scale multiple EEG corpora, which enhances its ability to generalize under limited and noisy pre-seizure data conditions.
While we focus our evaluation on the seizure task, generalization to other tasks remains a challenge of future work.
Regarding potential negative societal impacts, we recognize key risks such as bias and system malfunction. Specifically, models trained on EEG data from specific demographic groups may exhibit biased performance when applied to broader populations, potentially leading to unequal diagnostic accuracy. Additionally, miscalibrated early seizure prediction could lead to false alarms, causing unnecessary interventions or patient distress. 
These challenges highlight promising directions for future work, where incorporating fairness assessments, demographic audits, and human-in-the-loop strategies can lead to more robust and trustworthy EEG-based systems.

\newpage
\section*{Acknowledgments}
The authors would like to sincerely thank the anonymous reviewers for their valuable comments and helpful suggestions.
This work was supported by
JST BOOST, Japan Grant Number JPMJBS2402, “Program for Leading Graduate Schools” of The University of Osaka, 
JSPS KAKENHI Grant-in-Aid for Scientific Research Number JP24K20778,    
JST CREST JPMJCR23M3, 
JST START JPMJST2553, 
JST CREST JPMJCR20C6,
the Japan Science and Technology Agency (JST) AIP Acceleration Research (JPMJCR24U2, TY), 
Japan Agency for Medical Research and Development (AMED, JP19dm0307103, HK, JP24wm0625207, TY), 
the Japan Science and Technology Agency (JST) Moonshot R \&D-MILLENNIA Program (JPMJMS2012, TY), 
AIP Acceleration Research (JPMJCR24U2, TY), 
JST K Program JPMJKP25Y6, 
JST COI-NEXT JPMJPF2009, 
JST COI-NEXT JPMJPF2115,
Future Social Value Co-Creation Project - The University of Osaka.

\bibliographystyle{ACM-Reference-Format}
\bibliography{BIB/main}

\appendix
    \section*{Appendix}
     \label{sec:appendix}
     \begin{figure*}[h]
  \centering
  \includegraphics[width=1.0\linewidth]{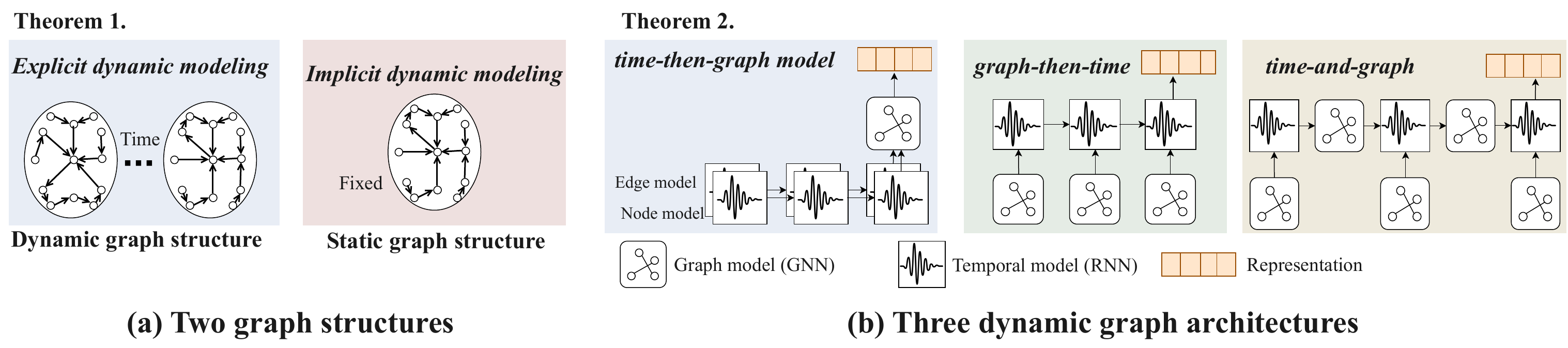}
  \caption{
  We compare (a) explicit dynamic modeling with implicit dynamic modeling and (b) time-then-graph with graph-then-time and graph-and-time architectures.
  }
  \label{fig: illustration}
    \vspace{-0.4em}
\end{figure*}

\section{Related Work}
    \label{sec:related}
    \textbf{Automated Seizure Analysis.}
The automated detection or prediction of seizures has been a long-standing challenge \citep{MetaLREC_AAAI2024, kotoge2024splitsee}. 
Deep learning has shown great achievements in automating EEG feature extraction and detection, using convolutional neural networks (CNNs) \citep{cnndetection1, CNN-LSTM, seizurenet, DenseCNN}, RNN-based models \citep{rnndetection1, CNN-LSTM, Rasheed}, Transformers \citep{TSTCC_ijcai21, BIOT_NEURIPS2023, LaBraM_iclr2024, MMM_NEURIPS2023}, and brain-inspired models \citep{SNN4, SNN3, SNN5, SNN2}.

\noindent\textbf{Spatial Relationships in EEG Networks.}
A seizure is fundamentally a network disease, and detection typically relies on the ability to determine abnormalities in EEG channels \citep{PNAS2014, NatNeuroscience2021, brain_dynamics}.
Many multi-channel methods have been proposed for capturing spatial information in channels \citep{LaBraM_iclr2024, MMM_NEURIPS2023, NEURIPS2023_Brant, KDD24_EEG2Rep}. 
Among them, recent studies have proposed GNNs to capture further the non-Euclidean structure of EEG electrodes and the connectivity in brain networks \citep{pmlr-seizure, ASTGCN, Cybernetics, BrainNet_KDD22, TNSRE22, EEG-GAT}.
These methods form EEGs as graphs, embedding each channel into the nodes and learning spatial graph representations \citep{EEG-GNN,GNN_AAAI23}. 
However, they do not explicitly model temporal relationships, relying instead on convolutional filters or conventional linear projections for node embeddings. 

\noindent
\textbf{Dynamic GNNs for EEG Modeling.}
Dynamic GNN is effective in learning temporal graph dynamics, achieving promising results in tasks such as dynamic link prediction \citep{iclr2024freedyg}, node classification \citep{KDD24_SEAN}, and graph clustering \citep{iclr2024_TGC}. 
Recently, some studies have focused on \DyGNN aimed to enhance temporal and graph representations for EEG-based seizure modeling. 
\citet{GNN_ICLR22} proposes a \timeandgraph model 
, which uses frequency features from FFT as node features and applies GNN and RNN processing simultaneously for each sliding window. \citet{MBrain_KDD23} adopts a \graphthentime model that combines GCN and RNN for seizure detection.
However, these studies construct static graphs with fixed structures across temporal learning. 
\citet{BiLSTM-GCNNet} propose \TTG model, BiLSTM-GCNNet, which uses RNNs to construct static graph from node feature.
GRAPHS4MER \citep{graphs4mer} is also proposed as a \TTG method. 
This work learns dynamic graphs using an internal graph structure learning model. 
Similarly, both \cite{asadzadeh2022accurate} and \cite{li2022graph} internally learn dynamic graph structures.
However, its input for graph structure learning is still based on the static Euclid distance or similarity of the entire data sample (e.g., 12 or 60 seconds) rather than individual snapshots. 
Our work differs by defining dynamic graph structures that more effectively capture the temporal evolution of brain connectivity in EEGs.

\section{Graph Isomorphism and 1-WL Test}\label{Tools}
\textbf{Graph isomorphism} refers to the problem of determining whether two graphs are structurally identical, meaning there exists a one-to-one correspondence between their nodes and edges. This is a crucial challenge in graph classification tasks, where the goal is to assign labels to entire graphs based on their structures. A model that can effectively differentiate non-isomorphic graphs is said to have high expressiveness, which is essential for accurate classification. In many cases, graph classification models like GNNs rely on graph isomorphism tests to ensure that structurally distinct graphs receive different embeddings, which improves the model's ability to classify graphs correctly. 

\noindent \textbf{1-Weisfeiler-Lehman (1-WL) test} is a widely used graph isomorphism test that forms the foundation of many GNNs. In the 1-WL framework, each node's representation is iteratively updated by aggregating information from its neighboring nodes, followed by a hashing process to capture the structural patterns of the graph. GNNs leveraging this concept, such as Graph Convolutional Networks (GCNs) and Graph Attention Networks (GATs), essentially perform a similar neighborhood aggregation, making them as expressive as the 1-WL test in distinguishing non-isomorphic graphs \citep{GIN_ICLR19}.  
Modern GNN architectures adhere to this paradigm, making the 1-WL a standard baseline for GNN expressivity.
In our work, we also use 1-WL-based GNNs, leveraging their proven expressiveness for dynamic brain graph modeling.

\section{Proofs of Expressivity Analysis}

\subsection{Expressiveness with Node and Edge Representations}
\label{proof_representaiton}

\begin{restatable}{lemma}{lemma:necessity_of_node_representations}[Necessity of Node Representations]
\label{lemma:necessity_of_node_representations}
\textbf{Edges alone} \citep{time-then-graph} \textbf{are insufficient} to uniquely distinguish certain temporal EEG graphs. Specifically, there exist pairs of temporal EEG graphs that have identical edge features across all time steps but different node features, making them indistinguishable based solely on edge representations. Therefore, \textbf{incorporating node representations is necessary} to achieve full expressiveness in EEG graph classification tasks.
\end{restatable}
\begin{proof}
Given \(\mathcal{G}^{(1)}\) and \(\mathcal{G}^{(2)}\), with the same sets of \(\mathcal{V}\) and \(\mathcal{E}\), 
for $\forall$ \(t\), the edge features satisfy:
$
a^{(1)}_{i,j,t} = a^{(2)}_{i,j,t} \quad \forall (v_i, v_j) \in \mathcal{E}, \forall t \in \{1, 2, \dots, T\}.
$
However, suppose there exists at least one node \(v_k \in \mathcal{V}\) and one time step \(t'\) such that:
$
\bm{x}^{(1)}_{k,t'} \neq \bm{x}^{(2)}_{k,t'}.
$
Since $a^{(1)}_{i,j,t} = a^{(2)}_{i,j,t}$, any GNN architecture that relies solely on edge features will produce identical embeddings for \(\mathcal{G}^{(1)}\) and \(\mathcal{G}^{(2)}\), $\forall t$. 
\end{proof}

\begin{restatable}{lemma}{expressiveness}[Expressiveness with Node and Edge Representations]
\label{lemma:expressiveness_with_node_and_edge}
When both node and edge representations are incorporated, a GNN can uniquely distinguish any pair of temporal EEG graphs that differ in either node features or edge features at any time step, provided the GNN is sufficiently expressive (e.g., 1-WL GNN). 
\end{restatable}

\begin{proof}
Given \(\mathcal{G}^{(1)} = (\mathcal{A}^{(1)}, \mathcal{X}^{(1)})\) and \(\mathcal{G}^{(2)} = (\mathcal{A}^{(2)}, \mathcal{X}^{(2)})\), suppose they differ in at least one node feature or edge feature at some time step \(t\). 
An expressive GNN can produce different embeddings for these graphs by capturing the differences in node and/or edge features. Specifically:

\begin{enumerate}
    \item If \(\mathbf{X}^{(1)}_{:, t} \neq \mathbf{X}^{(2)}_{:, t}\) for some \(t\), then the node embeddings \(h^{(1)}_{i,t}\) and \(h^{(2)}_{i,t}\) will differ for at least one node \(v_i\).
    \item If \(\bf{a}^{(1)}_{i,j,t} \neq \bf{a}^{(2)}_{i,j,t}\) for some edge \((v_i, v_j)\) and some \(t\), then the edge embeddings \(h^{(1)}_{ij,t}\) and \(h^{(2)}_{ij,t}\) will differ for that edge.
\end{enumerate}
Since the GNN aggregates information from both node and edge embeddings, any difference in either will propagate through the network, resulting in distinct final representations \(\mathbf{Z}^{(1)}\) and \(\mathbf{Z}^{(2)}\). Thus, the GNN can uniquely distinguish between \(\mathcal{G}^{(1)}\) and \(\mathcal{G}^{(2)}\).
\end{proof}

\subsection{Implicit and Explicit Dynamic Modeling}
\label{proof_implicit_explicit}
\implicitexplicit*

\begin{proof}
Let $\mathcal{F}_{\mathrm{implicit}}$ be the family of functions (e.g., 
representations or classifiers) expressible by an \emph{implicit} dynamic graph model 
that compresses all adjacency snapshots 
$\{\mathbf{A}_{:,:,t}\}_{t=1}^{T}$ into a single adjacency matrix 
$\widehat{\mathbf{A}}$, and let $\mathcal{F}_{\mathrm{explicit}}$ be the family of functions 
expressible by an \emph{explicit} dynamic graph model that uses the full 
time-varying adjacency $\{\mathbf{A}_{:,:,t}\}_{t=1}^{T}$.
We prove
\[
  \mathcal{F}_{\mathrm{implicit}} \;\subseteq\; 
  \mathcal{F}_{\mathrm{explicit}}
  \quad \text{and} \quad
  \mathcal{F}_{\mathrm{implicit}} \;\neq\; 
  \mathcal{F}_{\mathrm{explicit}},
\]
i.e.,
\[
  \mathcal{F}_{\mathrm{implicit}} \;\subset; 
  \mathcal{F}_{\mathrm{explicit}}.
\]
\textbf{(1) Subset Relationship.}\ 
It can be interpreted as any implicit dynamic model aggregates the set of adjacency matrices 
$\{\mathbf{A}_{:,:,t}\}_{t=1}^{T}$ into a single, static $\widehat{\mathbf{A}}$ 
by some function $f(\cdot)$ (e.g., averaging, summation, thresholding):
\[
  \widehat{\mathbf{A}} \;=\; 
  f\!\bigl(\{\mathbf{A}_{:,:,t}\}_{t=1}^{T}\bigr).
\]
Such a model uses $\widehat{\mathbf{A}}$ for all time steps, disregarding 
changes in edges across $t$. In contrast, an explicit dynamic model 
directly processes the full sequence 
$\{\mathbf{A}_{:,:,t}\}_{t=1}^{T}$. 
To see that $\mathcal{F}_{\mathrm{implicit}}$ is contained in 
$\mathcal{F}_{\mathrm{explicit}}$, observe that any function 
$f_{\mathrm{imp}} \in \mathcal{F}_{\mathrm{implicit}}$ 
can be mimicked by an explicit model that simply 
\emph{ignores} the time-specific variability in adjacency and 
substitutes $f(\{\mathbf{A}_{:,:,t}\}) = \widehat{\mathbf{A}}$ 
as if it were the adjacency for every $t$. 
Hence,
\[
  \forall\,f_{\mathrm{imp}}\in\mathcal{F}_{\mathrm{implicit}},\quad
  f_{\mathrm{imp}}\in\mathcal{F}_{\mathrm{explicit}},
\]
implying
\[
  \mathcal{F}_{\mathrm{implicit}}
  \;\subseteq\;
  \mathcal{F}_{\mathrm{explicit}}.
\]

\textbf{(2) Strict Separation by Counterexample.}\ 
To show that $\mathcal{F}_{\mathrm{implicit}} \neq \mathcal{F}_{\mathrm{explicit}}$, 
we construct two temporal EEG graphs, 
$\mathcal{G}^{(1)}=\{\mathbf{X}_{:,t}, \mathbf{A}_{:,:,t}^{(1)}\}_{t=1}^{T}$ 
and 
$\mathcal{G}^{(2)}=\{\mathbf{X}_{:,t}, \mathbf{A}_{:,:,t}^{(2)}\}_{t=1}^{T}$, 
such that
\begin{enumerate}
    \item Their node features match at every time step:
    \[
      \mathbf{X}_{:,t}^{(1)} \;=\; \mathbf{X}_{:,t}^{(2)} 
      \quad \forall\,t\in\{1,\dots,T\}.
    \]
    \item Their adjacency differs \emph{only} at one time step $t'$:
    \[
      \mathbf{A}_{:,:,t'}^{(1)} 
      \;\neq\; 
      \mathbf{A}_{:,:,t'}^{(2)},
      \quad
      \text{and}
      \quad
      \mathbf{A}_{:,:,t}^{(1)} \;=\; \mathbf{A}_{:,:,t}^{(2)}
      \quad \forall\,t\neq t'.
    \]
\end{enumerate}
Under \emph{implicit} dynamic modeling, the function 
$f(\{\mathbf{A}_{:,:,t}\}_{t=1}^T)$ may yield the same compressed adjacency 
$\widehat{\mathbf{A}}$ for both $\mathcal{G}^{(1)}$ and $\mathcal{G}^{(2)}$. 
Hence, any implicit model would treat the two graphs \emph{identically}, 
failing to separate them because the single $\widehat{\mathbf{A}}$ 
cannot preserve the difference at $t'$.

In contrast, the \emph{explicit} dynamic model processes 
$\{\mathbf{A}_{:,:,t}^{(k)}\}_{t=1}^T$ without discarding the 
time-step-specific variations. Consequently, it \emph{does} see 
$\mathbf{A}_{:,:,t'}^{(1)} \neq \mathbf{A}_{:,:,t'}^{(2)}$, 
and can therefore distinguish $\mathcal{G}^{(1)}$ from $\mathcal{G}^{(2)}$. 
Thus, there exist inputs in the domain of temporal EEG graphs 
that are \emph{indistinguishable} by any implicit dynamic model 
but \emph{distinguishable} by an explicit dynamic model. 
This proves
\[
  \mathcal{F}_{\mathrm{implicit}} \;\neq\; 
  \mathcal{F}_{\mathrm{explicit}},
\]
and in view of the subset argument, we conclude 
\[
  \mathcal{F}_{\mathrm{implicit}} \;\subset\; 
  \mathcal{F}_{\mathrm{explicit}},
\]
explicit dynamic modeling is strictly more expressive than implicit dynamic modeling.
\end{proof}

\subsection{\TAG and \TTG}
\label{proof_ttg}
\tagttg*
\citet{time-then-graph} prove that a \TTG representation that outputs the same
embeddings as an arbitrary \TAG representation.
Thus, 
\TTG is as expressive as
\TAG.
To prove Lemma \ref{tagttg} we also provide an EEG graph classification task where any \TAG
representation will fail while a \TTG would work, which then, added
to the previous result, proves that \TTG is strictly more expressive
than \TAG.

\begin{figure}[tb]
\begin{center}
\centerline{\includegraphics[width=0.8\columnwidth]{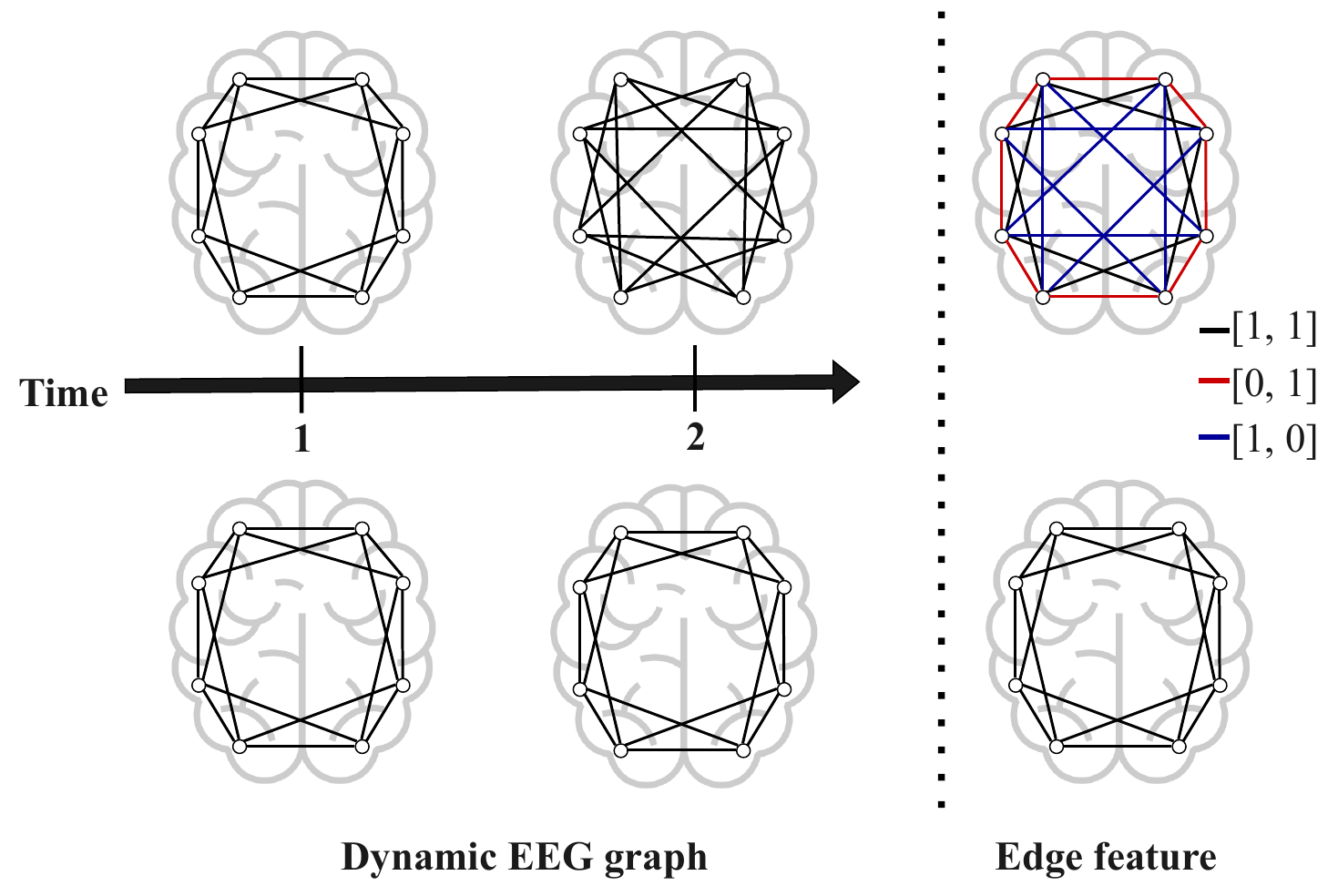}}
\caption{
    {
        \bf A synthetic EEG task where only \TTG is expressive.
    }
    The top and bottom 2-time temporal graphs on the left side have snapshots of
    different structure at time $t_{2}$ (denote by $\mathcal{C}_{8, 2}$ and $
        \mathcal{C}_{8, 1}
    $).
    The top and bottom temporal graphs on the left show different dynamic-graph structures. 
The right side shows their aggregated versions, where edge attributes indicate whether they existed (1) or not (0) over time, using different colors. 
The goal is to distinguish the structural differences between the top and bottom graphs.
\TAG has the same node representation neighbors in both snapshots, indistinguishable. \TTG aggregates the dynamic graphs into different node representations and succeeds in distinguishing them.
}
\label{fig:dyncsl}
\end{center}
\vskip -0.3in
\end{figure}

\begin{proof}
We now propose a synthetic EEG task, 
whose temporal graph is illustrated in Figure \ref{fig:dyncsl}.
The goal is to differentiate the topologies between two 2-step temporal graphs.
Each snapshot is a static EEG graph
with 8 attributed nodes, denoted as $\mathcal{C}_{8, 1}$ and $ \mathcal{C}_{8, 2}$.

Two temporal graphs differ in their second time step $
    t_{2}
$.
If the graphs have the same features, any 1-WL
GNN will output the same representations for both $\mathcal{C}_{8, 1}$ and $\mathcal{C}_{8, 2}$.
We use $\bf{a}^{(\text{top})}$ to represent the adjacency matrix of dynamics in
the top left of Figure \ref{fig:dyncsl}, and $\bf{a}^{(\text{btm})}$ for dynamics in
the bottom left of Figure \ref{fig:dyncsl}.
Note that $\mathbf{X}^{(\text{top})} = \mathbf{X}^{(\text{btm})}$ since the temporal
graph has the same features.

Hence, for a \TAG representation:
\begin{align*}
\text{GNN}_\text{in}^{L}\big(
    \mathbf{X}^{(\text{top})}_{:, 1}, \mathbf{A}^{(\text{top})}_{:, :, 1}
\big) &= \text{GNN}_\text{in}^{L}\big(
    \mathbf{X}^{(\text{top})}_{:, 2}, \mathbf{A}^{(\text{top})}_{:, :, 2}
\big)   
\\= \text{GNN}_\text{in}^{L}\big(
    \mathbf{X}^{(\text{btm})}_{:, 1}, \mathbf{A}^{(\text{btm})}_{:, :, 1}
\big) 
&= \text{GNN}_\text{in}^{L}\big(
    \mathbf{X}^{(\text{btm})}_{:, 2}, \mathbf{A}^{(\text{btm})}_{:, :, 2}
\big),
\\
\text{GNN}_\text{rc}^{L}\big(\mathbf{H}^{(\text{top})}_{:, 0}, \mathbf{A}^{(\text{top})}_{:, :, 1}\big) &= 
\text{GNN}_\text{rc}^{L}\big(\mathbf{H}^{(\text{btm})}_{:, 0}, \mathbf{A}^{(\text{btm})}_{:, :, 1}
\big),  \\
\text{GNN}_\text{rc}^{L}\big( \mathbf{H}^{(\text{top})}_{:, 1}, \mathbf{A}^{(\text{top})}_{:, :, 2} \big) 
&=\text{GNN}_\text{rc}^{L}\big( \mathbf{H}^{(\text{btm})}_{:, 1}, \mathbf{A}^{(\text{btm})}_{:, :, 2} \big).
\end{align*}
Then, when we apply Equation (\ref{eqn:timeandgraph}) at the first time step, 
for the top graph:
\begin{align*}
\bf{h}^{(\text{top})}_{i, 1} &= \text{Cell}\bigg(
    \Big[\text{GNN}_\text{in}^{L}\big(\mathbf{X}^{(\text{top})}_{:, 1}, \mathbf{A}^{(\text{top})}_{:, :, 1}\big)\Big]_{i},
    \Big[\text{GNN}_\text{rc}^{L}\big(\bf{h}^{(\text{top})}_{:, 0}, \mathbf{A}^{(\text{top})}_{:, :, 1}\big)\Big]_{i}
\bigg),
\end{align*}
for the bottom graph:
\begin{align*}
\bf{h}^{(\text{btm})}_{i, 1} &= \text{Cell}\bigg(
    \Big[\text{GNN}_\text{in}^{L}\big(\mathbf{X}^{(\text{btm})}_{:, 1}, \mathbf{A}^{(\text{btm})}_{:, :, 1}\big)\Big]_{i},
    \Big[\text{GNN}_\text{rc}^{L}\big(\bf{h}^{(\text{btm})}_{:, 0}, \mathbf{A}^{(\text{btm})}_{:, :, 1}\big)\Big]_{i}
\bigg).
\end{align*}
Since $\mathbf{X}^{(\text{top})} = \mathbf{X}^{(\text{btm})}$, $\mathbf{A}^{(\text{top})}_{:, :, 1} = \mathbf{A}^{(\text{btm})}_{:, :, 1}$, and $\bf{h}^{(\text{top})}_{:, 0} = \bf{h}^{(\text{btm})}_{:, 0}$, \\
we have:
$
\bf{h}^{(\text{top})}_{i, 1} = \bf{h}^{(\text{btm})}_{i, 1}.
$

\noindent For the second time step:
\begin{align*}
\bf{h}^{(\text{top})}_{i, 2} &= \text{Cell}\bigg(
    \Big[\text{GNN}_\text{in}^{L}\big(\mathbf{X}^{(\text{top})}_{:, 2}, \mathbf{A}^{(\text{top})}_{:, :, 2}\big)\Big]_{i},
    \Big[\text{GNN}_\text{rc}^{L}\big(\bf{h}^{(\text{top})}_{:, 1}, \mathbf{A}^{(\text{top})}_{:, :, 2}\big)\Big]_{i}
\bigg), \\
\bf{h}^{(\text{btm})}_{i, 2} &= \text{Cell}\bigg(
    \Big[\text{GNN}_\text{in}^{L}\big(\mathbf{X}^{(\text{btm})}_{:, 2}, \mathbf{A}^{(\text{btm})}_{:, :, 2}\big)\Big]_{i},
    \Big[\text{GNN}_\text{rc}^{L}\big(\bf{h}^{(\text{btm})}_{:, 1}, \mathbf{A}^{(\text{btm})}_{:, :, 2}\big)\Big]_{i}
\bigg).
\end{align*}
Despite $\mathbf{A}^{(\text{top})}_{:, :, 2} \neq \mathbf{A}^{(\text{btm})}_{:, :, 2}$, the 1-WL GNN will output the same representations $\bf{h}^{(\text{top})}_{i, 2} = \bf{h}^{(\text{btm})}_{i, 2}$ for both $\mathcal{C}_{8, 1}$ and $\mathcal{C}_{8, 2}$. 

\noindent Therefore:
\begin{align*}
\mathbf{Z}^{(\text{top})} = \bf{h}^{(\text{top})}_{i, 2} = \bf{h}^{(\text{btm})}_{i, 2} = \mathbf{Z}^{(\text{btm})}.
\end{align*}
\textbf{Thus, time-and-graph will output the same final representation $\mathbf{Z}^{(\text{top})} =  \mathbf{Z}^{(\text{btm})}$ for two different temporal graphs in Figure \ref{fig:dyncsl}.}

For the time-then-graph representation, we apply Equation (\ref{eqn:timethengraph}).
First, for the node representations:

\begin{align*}
\bf{h}^\text{node(top)}_i &= \text{RNN}^\text{node}\big( \mathbf{X}^{(\text{top})}_{i, \leq 2} \big) 
= \text{RNN}^\text{node}\big( [\bf{x}^{(\text{top})}_{i, 1}, \bf{x}^{(\text{top})}_{i, 2}] \big), \\
\bf{h}^\text{node(btm)}_i &= \text{RNN}^\text{node}\big( \mathbf{X}^{(\text{btm})}_{i, \leq 2} \big) 
= \text{RNN}^\text{node}\big( [\bf{x}^{(\text{btm})}_{i, 1}, \bf{x}^{(\text{btm})}_{i, 2}] \big).
\end{align*}
Since $\mathbf{X}^{(\text{top})} = \mathbf{X}^{(\text{btm})}$, we have $\bf{h}^\text{node(top)}_i = \bf{h}^\text{node(btm)}_i$ for all nodes $i$.

\noindent Now, for the edge representations:

\begin{align*}
\bf{h}^\text{edge(top)}_{i,j} &= \text{RNN}^\text{edge}\big( \mathbf{A}^{(\text{top})}_{i,j,\leq 2} \big) 
= \text{RNN}^\text{edge}\big( [\bf{a}^{(\text{top})}_{i,j,1}, \bf{a}^{(\text{top})}_{i,j,2}] \big), \\
\bf{h}^\text{edge(btm)}_{i,j} &= \text{RNN}^\text{edge}\big( \mathbf{A}^{(\text{btm})}_{i,j,\leq 2} \big) = \text{RNN}^\text{edge}\big( [\bf{a}^{(\text{btm})}_{i,j,1}, \bf{a}^{(\text{btm})}_{i,j,2}] \big).
\end{align*}
Here, $\bf{a}^{(\text{top})}_{i,j,\leq 2} \neq \bf{a}^{(\text{btm})}_{i,j,\leq 2}$ for some $(i,j)$ pairs, because the graph structures differ at $t_2$. Therefore, $\bf{h}^\text{edge(top)}_{i,j} \neq \bf{h}^\text{edge(btm)}_{i,j}$ for these pairs.
Finally, we apply the GNN:

\begin{align*}
\mathbf{Z}^{(\text{top})} &= \text{GNN}^L\big( \mathbf{H}^\text{node(top)}, \mathbf{H}^\text{edge(top)} \big), \\
\mathbf{Z}^{(\text{btm})} &= \text{GNN}^L\big( \mathbf{H}^\text{node(btm)}, \mathbf{H}^\text{edge(btm)} \big).
\end{align*}
Since $\bf{h}^\text{edge(top)} \neq \bf{h}^\text{edge(btm)}$, and 1-WL GNNs can distinguish graphs with different edge attributes, we have:

\begin{equation}
\mathbf{Z}^{(\text{top})} \neq \mathbf{Z}^{(\text{btm})}.
\end{equation}
\textbf{Thus, time-then-graph outputs different final representations $\mathbf{Z}^{(\text{top})} \neq \mathbf{Z}^{(\text{btm})}$ for the two temporal graphs in Figure \ref{fig:dyncsl}, successfully distinguishing them.}

Finally, we conclude:
\begin{enumerate}
\item
The \TTG is at least as expressive as the \TAG;
\item
The \TTG can distinguish temporal graphs not distinguishable by
\TAG.
\end{enumerate}
Thus, \TTG is strictly more expressive than \TAG.
More precisely,
\begin{align*}
\text{{\TAG}} \exprlt_{\mathbb{T}_{n, T, \theta}} \text{{\TTG}},
\end{align*}
concluding our proof.

\end{proof}

\section{Computational Complexity}
\label{sec:complexity}

In this section, we provide the computational complexities of the three architectures: \emph{graph-then-time}, \emph{time-and-graph}, and \emph{time-then-graph}. We demonstrate that the \emph{time-then-graph} architecture has the lowest computational complexity among them.

Let \( T \) be the number of time steps, \( V \) be the number of nodes, \( E_t \) be the number of edges at time \( t \), \( \sum_{t} E_t \) be the total number of edges across all time steps, \( E_{\text{agg}} \) be the number of edges in the aggregated graph (i.e., the union of all edges across time steps), and \( d \) be the dimension of the node and edge representations.

\subsection{\emph{graph-then-time} Architecture}

In the \emph{graph-then-time} architecture, at each time step \( t \), a GNN is applied to the snapshot graph \( (\mathbf{X}_{:, t}, \bf{a}_{:, :, t}) \) to capture spatial relationships. Subsequently, an RNN processes the node embeddings over time to capture temporal dependencies.

The computational complexity per time step \( t \) is dominated by:

\[
\mathcal{O}\left( V d^{2} + E_t d \right),
\]

where \( V d^{2} \) accounts for node-wise transformations (e.g., linear layers), and \( E_t d \) accounts for message passing over edges.

Over all time steps, the total complexity for the GNN computations is:

\[
\mathcal{O}\left( T V d^{2} + \sum_{t=1}^{T} E_t d \right).
\]

The RNN processes the node embeddings over time with complexity:

\[
\mathcal{O}\left( V T d^{2} \right).
\]

Therefore, the overall computational complexity of the \emph{graph-then-time} architecture is:

\begin{equation}
\label{eq:complexity_graph_then_time}
\mathcal{O}\left( V T d^{2} + \sum_{t=1}^{T} E_t d \right).
\end{equation}

\subsection{\emph{time-and-graph} Architecture}

In the \emph{time-and-graph} architecture, temporal dependencies are integrated into the GNN computations. At each time step \( t \), two GNNs are applied:

\begin{itemize}
    \item \(\text{GNN}_{\text{in}}^{L}\) processes the current snapshot inputs \( (\mathbf{X}_{:, t}, \bf{a}_{:, :, t}) \).
    \item \(\text{GNN}_{\text{rc}}^{L}\) processes the representations from the previous time step \( (\bf{h}_{:, t-1}, \bf{a}_{:, :, t}) \).
\end{itemize}

The computational complexity per time step \( t \) is:

\[
\mathcal{O}\left( V d^{2} + E_t d^{2} \right),
\]

due to the node-wise transformations and edge-wise message passing with updated representations.

Over all time steps, the total complexity for the GNN computations is:

\[
\mathcal{O}\left( T V d^{2} + \sum_{t=1}^{T} E_t d^{2} \right).
\]

The RNN (or any recurrent unit) further processes the node embeddings with complexity:

\[
\mathcal{O}\left( V T d^{2} \right).
\]

Therefore, the overall computational complexity of the \emph{time-and-graph} architecture is:

\begin{equation}
\label{eq:complexity_time_and_graph}
\mathcal{O}\left( V T d^{2} + \sum_{t=1}^{T} E_t d^{2} \right).
\end{equation}

\subsection{\emph{time-then-graph} Architecture}

In the \emph{time-then-graph} architecture, temporal evolutions of node and edge attributes are modeled first using sequence models (e.g., RNNs). A GNN is then applied to the resulting static graph with aggregated temporal information.

The computational complexities are as follows:

\paragraph{Node Sequence Modeling.}
For each node \( i \in \mathcal{V} \), an RNN processes its temporal features \( \mathbf{X}_{i, \leq T} \):
\[
\mathcal{O}\left( V T d^{2} \right).
\]

\paragraph{Edge Sequence Modeling.}
For each edge \( (i, j) \in \mathcal{E}_{\text{agg}} \), an RNN processes its temporal adjacency features \( \bf{a}_{i, j, \leq T} \):
\[
\mathcal{O}\left( E_{\text{agg}} T d^{2} \right).
\]

\paragraph{GNN over Aggregated Graph.}
A GNN is applied once to the static graph with updated node and edge representations:
\[
\mathcal{O}\left( V d^{2} + E_{\text{agg}} d^{2} \right).
\]

Therefore, the overall computational complexity of the \emph{time-then-graph} architecture is:
\begin{equation}
\label{eq:complexity_time_then_graph}
\mathcal{O}\left( \left( V + E_{\text{agg}} \right) T d^{2}  \right).
\end{equation}

\subsection{Comparison of Complexities}

To compare the computational complexities, we consider the dominant terms in Equations~\eqref{eq:complexity_graph_then_time}, \eqref{eq:complexity_time_and_graph}, and \eqref{eq:complexity_time_then_graph}.

\begin{itemize}
    \item \textbf{graph-then-time}:

    \[
    \mathcal{O}\left( V T d^{2} + \sum_{t=1}^{T} E_t d \right).
    \]

    \item \textbf{time-and-graph}:

    \[
    \mathcal{O}\left( V T d^{2} + \sum_{t=1}^{T} E_t d^{2} \right).
    \]

    \item \textbf{time-then-graph}:

    \[
    \mathcal{O}\left( \left( V + E_{\text{agg}} \right) T d^{2} \right).
    \]
\end{itemize}

By comparing these computational complexities, the \textbf{time-then-graph} method is superior under the aggregated number of edges $E_{\text{agg}}$ is smaller than the total sum of edges over all time steps, i.e., $E_{\text{agg}} \ll \sum_{t=1}^{T} E_t$.



\section{Experimental details and implementation}
\label{section: training} 

\noindent\textbf{Model training.} Training for all models was accomplished using the Adam optimizer \citep{adam2014} in PyTorch on NVIDIA A6000 GPU and Xeon Gold 6258R CPU.  

\noindent\textbf{Data augmentation.} 
During the training process, we applied the following data augmentation techniques, following prior studies \citep{GNN_ICLR22, TSTCC_ijcai21}: randomly scaling the amplitude of the raw EEG signals by a factor between 0.8 and 1.2.

\noindent\textbf{Implementation details.}
We used binary cross-entropy as the loss function to train all models. 
The models were trained for 100 epochs with an initial learning rate of 1e-4. 
To enhance efficiency and sparsity, we set 
$\tau = 3$ and the top-3 neighbors’ edges were kept for each node. 
The dropout probability was 0 (i.e., no dropout). 
\method has two Mamba consisting of two stacked layers and a two-layer GCN with 64 hidden units, resulting in 114,794 trainable parameters. 
We release
GitHub repository (\url{https://github.com/Kotoge/EvoBrain}).

\noindent\textbf{Implementation of baselines.}
For baselines, DCRNN \citep{GNN_ICLR22}, EvolveGCN \citep{AAAI20_EvolveGCN}, and LSTM \citep{LSTM}, we used the number of RNN and GNN layers and hidden units in our \method. 
For BIOT, we use the same model architecture described in \citet{BIOT}, i.e., four Transformer layers with eight attention heads and 256-dimensional embedding.
For LaBraM, we used the official checkpoint provided by \cite{LaBraM_iclr2024}.
For CNN-LSTM, we use the same model architecture described in \citet{CNN-LSTM}, i.e., two stacked convolutional layers (32 $3 \times 3$ kernels), one max-pooling layer ($2 \times 2$), one fully-connected layer (output neuron = 512), two stacked LSTM layers (hidden size = 128), and one fully connected layer.
Table \ref{tab:trainable_params} shows a comparison of trainable parameters, with our \method achieving the best performance using the relatively few parameters.

\begin{table}[h]
\centering
\caption{Comparison of trainable parameters.}
\resizebox{\linewidth}{!}{%
\begin{tabular}{l|c|c|c|c|c|c|c|c|c}
\toprule
\textbf{} & \method & GRU-GCN & DCRNN & EvolveGCN & EEGPT & LaBraM & BIOT & CNN-LSTM & LSTM \\ \midrule
\makecell{Trainable \\ Parameters} & 183,834 & 114,794 & 280,769 & 200,301 & 51,221,121 &5,803,137 & 3,187,201 & 5,976,033 & 536,641 \\ \bottomrule
\end{tabular}%
}
\label{tab:trainable_params}
\end{table}

\section{Data description}
Tabel \ref{tab:dataset_summary} shows the details of TUSZ dataset.
\label{data}
\begin{table*}[h]
\centering
\caption{Number of EEG data samples and patients in the train, validation, and test sets on TUSZ dataset. Train, validation, and test sets consist of distinct patients.}
\label{tab:dataset_summary}
\resizebox{\textwidth}{!}{\begin{tabular}{l|l|c|c|c|c|c|c}
\toprule
&\multirow{2}{*}{\begin{tabular}[c]{@{}l@{}}EEG Input \\ Length \\ (Secs)\end{tabular}} & \multicolumn{2}{c|}{Train Set} & \multicolumn{2}{c|}{Validation Set}  & \multicolumn{2}{c}{Test Set}                                                                                                        \\ \cline{3-8} Task&   & \begin{tabular}[c]{@{}c@{}}EEG samples \\ \% 
(Pre-) Seizure\end{tabular} & \begin{tabular}[c]{@{}c@{}}Patients\end{tabular} & \begin{tabular}[c]{@{}c@{}}EEG samples \\ \% 
(Pre-) Seizure\end{tabular} & \begin{tabular}[c]{@{}c@{}}Patients\end{tabular} & \begin{tabular}[c]{@{}c@{}}EEG samples \\ \% 
(Pre-) Seizure\end{tabular} & \begin{tabular}[c]{@{}c@{}}Patients\end{tabular} \\ \midrule
\multirow{2}{*}{\begin{tabular}[c]{@{}l@{}}Seizure \\ Detection\end{tabular}} & 60-s                                                                                  & \begin{tabular}[c]{@{}c@{}}38,613 (9.3\%)\end{tabular}         & \begin{tabular}[c]{@{}c@{}}530 \end{tabular}          & \begin{tabular}[c]{@{}c@{}}5,503 (11.4\%)\end{tabular}         & \begin{tabular}[c]{@{}c@{}}61 \end{tabular}           & \begin{tabular}[c]{@{}c@{}}8,848  (14.7\%)\end{tabular}         & \begin{tabular}[c]{@{}c@{}}45\end{tabular}           \\ 
& 12-s                                                                                  & \begin{tabular}[c]{@{}c@{}}196,646 (6.9\%)\end{tabular}         & \begin{tabular}[c]{@{}c@{}}531 \end{tabular}          & \begin{tabular}[c]{@{}c@{}}28,057  (8.7\%)\end{tabular}         & \begin{tabular}[c]{@{}c@{}}61 \end{tabular}           & \begin{tabular}[c]{@{}c@{}}44,959 (10.9\%)\end{tabular}        & \begin{tabular}[c]{@{}c@{}}45 \end{tabular}           \\ \midrule  
\multirow{2}{*}{\begin{tabular}[c]{@{}l@{}}Seizure \\ Prediction\end{tabular}} & 60-s                                                                                  & \begin{tabular}[c]{@{}c@{}}7,550  (9.9\%)\end{tabular}         & \begin{tabular}[c]{@{}c@{}}530 \end{tabular}          & \begin{tabular}[c]{@{}c@{}}999  (12.0\%)\end{tabular}         & \begin{tabular}[c]{@{}c@{}}61 \end{tabular}           & \begin{tabular}[c]{@{}c@{}}1,277 (24.4\%)\end{tabular}         & \begin{tabular}[c]{@{}c@{}}45\end{tabular}           \\ 
& 12-s                                                                                  & \begin{tabular}[c]{@{}c@{}}40,716 (12.8\%)\end{tabular}         & \begin{tabular}[c]{@{}c@{}}531 \end{tabular}          & \begin{tabular}[c]{@{}c@{}}5,439  (16.0\%)\end{tabular}         & \begin{tabular}[c]{@{}c@{}}61 \end{tabular}           & \begin{tabular}[c]{@{}c@{}}6,956 (27.6\%)\end{tabular}        & \begin{tabular}[c]{@{}c@{}}45 \end{tabular}           
\\ \bottomrule
\end{tabular}}
\end{table*}

\section{Memory Efficiency}

\begin{table}[h]
\centering
\caption{GPU memory consumption during training and inference. EvoBrain offers competitive memory efficiency compared to baselines.}
\label{tab:gpu_memory}
\begin{tabular}{lcc}
\toprule
\textbf{Model} & \textbf{Training (MB)} & \textbf{Inference (MB)} \\
\midrule
EvoBrain       & 51.35                  & 46.64                    \\
GRU-GCN        & 54.61                  & 52.09                    \\
GRAPHS4MER     & 369.46                 & 93.02                    \\
DCRNN          & 21.10                  & 20.54                    \\
EvolveGCN      & 22.06                  & 20.07                    \\
\bottomrule
\end{tabular}
\end{table}
Figure 4 shows memory efficiency comparisons across dynamic GNNs in both training and inference settings with a batch size of 1. While DCRNN and EvolveGCN exhibit the lowest memory usage, EvoBrain achieves over 10× faster computation compared to these baselines, making it a compelling choice for time-critical applications. Among the three time-then-graph models, EvoBrain demonstrates the best balance between memory consumption and computational speed, with significantly lower memory requirements than GRAPHS4MER and comparable usage to GRU-GCN.


\end{document}